\documentclass[3p]{elsarticle}

\usepackage{natbib}
\usepackage{caption}
\usepackage{graphicx}
\usepackage{makecell}
\usepackage{subfig}
\usepackage{amsmath}
\usepackage{amssymb}
\usepackage{amsthm}
\usepackage{algorithm}
\usepackage{algorithmic}
\usepackage{threeparttable}
\usepackage{multirow}
\usepackage{bm}
\usepackage{cases}
\usepackage{array}
\usepackage{url}

\usepackage{textcomp}
\usepackage{xcolor}
\usepackage{booktabs}
\usepackage{float}
\usepackage{verbatim}

\newtheorem{definition}{Definition}
\newtheorem{corollary}{Corollary}

\newtheorem{theorem}{Theorem}

\newtheorem{lemma}{Lemma}
\newtheorem{proposition}{Proposition}

\journal{Artificial Intelligence}

\begin{document}
	
	\begin{frontmatter}
		\title{Multi-Party Multi-Objective Optimization as Consensus Search: Runtime Analysis of Cross-Party Recombination}
		
		\author[1]{Xiaolei Fang}
		\ead{202383250038@nuist.edu.cn}
		\author[1]{Peilan Xu\corref{cor1}}
		\ead{xpl@nuist.edu.cn}
		\author[2]{Wenjian Luo}
		\ead{luowenjian@hit.edu.cn}
		\cortext[cor1]{Corresponding author.}
		
		\address[1]{School of Artificial Intelligence, Nanjing University of Information Science and Technology, Nanjing 210044, China}
		\address[2]{Guangdong Provincial Key Laboratory of Novel Security Intelligence Technologies, Institute of Cyberspace Security, School of Computer Science and Technology, Harbin Institute of Technology, Shenzhen 518055, China.}
		
		\begin{abstract}
			Multi-party multi-objective optimization problems (MPMOPs) require consensus among autonomous decision makers and therefore differ from flattened many-objective formulations. Existing runtime theory for multi-objective evolutionary algorithms is largely tailored to single-party Pareto-front approximation and does not directly explain common-solution search in MPMOPs. We investigate cross-party recombination in two representative settings. On MP-JCG, a pseudo-Boolean benchmark with an explicit gap region, we prove that a payoff-guided mutation baseline faces a gap-crossing bottleneck requiring \(\Theta(n^2)\) expected fitness evaluations. In contrast, an analytical CPR-NSGA-II variant discovers both common Pareto-optimal solutions in \(O(n\log n)\) expected evaluations by directly assembling complementary prefix and suffix templates distributed across party populations. Comparing this with the flattened four-objective formulation F-JCG, our full-front coverage analysis illustrates the additional coverage burden introduced by flattening. For BPBOMST, the bi-party, two-objective-per-party specialization of the multi-party multi-objective minimum spanning tree problem, we develop a layered support-cover analysis. For each common Pareto objective vector, the symmetric average projection induces an auxiliary bi-objective MST instance, and suitable support representatives yield a \(2\lambda\)-common approximation cover with \(\lambda\in[1,2]\). We further derive an instance-parameterized expected runtime bound for a representative-pool CPR-NSGA-II variant using edge-union recombination and uniform repair. This bound separates the effects of local auxiliary-front filling, cross-party recombination shortcuts, and edge-union repair ambiguity. Overall, our analysis demonstrates that cross-party recombination can efficiently assemble complementary information from different parties to construct common solutions.
		\end{abstract}
		
		
	\end{frontmatter}
	
	\section{Introduction}
	\label{sec:introduction}
	
		Many real-world decision-making problems are inherently multi-party, in which multiple autonomous decision makers (DMs) operate over a shared decision space while pursuing distinct and often conflicting objectives. Representative examples include reservoir flood-control scheduling balancing flood safety, hydropower generation, and ecological protection \cite{chuntianThreepersonMultiobjectiveConflict2002}, cross-basin water diversion involving upstream suppliers, downstream consumers, and regulatory authorities \cite{zhangNegotiationBasedMultiObjectiveMultiParty2012a}, and UAV path planning in which operational efficiency must comply with regulatory safety constraints \cite{chenEvolutionaryBipartyMultiobjective2024}. Similar multi-party structures also arise in network design, where different agents or organizations may evaluate the same graph through different edge preferences, costs, or reliability requirements \cite{li2023multiagentMSTCover,darmann2011finding,darmann2016condorcet}. Across these settings, no single DM can fully impose a decision unilaterally, and acceptable solutions must instead be reached through interaction and compromise among multiple parties.
	
	Traditionally, such problems are often simplified by assuming a centralized DM that aggregates all party-wise objectives into a single multi-objective optimization problem (MOP). While this centralized modeling facilitates the application of classical multi-objective evolutionary algorithms (MOEAs), it neglects the autonomy and strategic independence of individual parties, thereby altering the underlying search target. Multi-party multi-objective optimization problems (MPMOPs) \cite{liuEvolutionaryApproachMultiparty2020} depart from this assumption by retaining party-wise objective vectors over a shared decision space, ensuring that each party evaluates the same solution through its own objectives. From this perspective, optimization shifts from merely approximating a Pareto front (PF) to identifying common solutions when they exist. Accordingly, a central target in MPMOPs is to identify common Pareto-optimal solutions, defined as the intersection of the individual Pareto sets of all DMs. This distinction highlights that MPMOPs must be analyzed as consensus-search problems rather than as ordinary PF-approximation tasks.
	
	Population-based metaheuristics have demonstrated strong capability in approximating well-distributed PFs for MOPs \cite{tian2021survey}. Accordingly, many existing evolutionary algorithms for MPMOPs are built by extending classical MOEAs, such as NSGA-II \cite{debFastElitistMultiobjective2002} and MOEA/D \cite{zhangMOEAMultiobjectiveEvolutionary2007}, and then incorporating additional multi-party interaction mechanisms. These mechanisms can be organized into three broad categories. First, some studies modify the selection stage to explicitly incorporate multi-party selection pressure, for example through multi-party non-dominated sorting based on party-specific ranks or dominance levels \cite{liuEvolutionaryApproachMultiparty2020,sheNewEvolutionaryApproach2021,she2023}, or through indicator-based selection using specially designed multi-party indicators \cite{songIndicatorBasedEvolutionary2024}. Second, multi-party information interaction can also be realized within a single population, either by alternating the active party over time \cite{changMultipartyMultiobjectiveOptimization2022,sun2025runtime} or by exploiting multi-party sorting information to guide recombination among individuals \cite{chen2025}. Third, multi-population collaborative frameworks maintain one population for each party and employ inter-population information exchange or recombination to construct candidate common solutions or accelerate their discovery \cite{zhangMultiPartyMultiObjectiveOptimization2025}. This line of work already suggests several algorithmic routes toward multi-party consensus search, but it still leaves open the central theoretical question of when and why cross-party interaction, especially cross-party recombination, can provably accelerate the discovery of common Pareto-optimal solutions.
	
	Despite the rapid development of evolutionary algorithms (EAs) for multi-party multi-objective optimization, the theoretical understanding remains limited, especially compared with the maturing runtime theory of classical EAs and MOEAs \cite{he2001drift,laumanns2002running}. A common simplification is to concatenate all party-wise objectives into a single many-objective optimization problem (MaOP). However, flattening changes not only the number of objectives but also the underlying search target. From the MaOP perspective, increasing the objective dimension is known to weaken Pareto-dominance discrimination, enlarge the set of mutually non-dominated solutions, and reduce the effectiveness of diversity-maintenance mechanisms, making population-based search increasingly difficult \cite{li2015many}. More specifically, recent runtime analyses show that NSGA-II can require exponential time for full-front coverage in many-objective settings \cite{zheng2024}, and that crowding-distance-based selection may become ineffective beyond two objectives on many-objective benchmark problems \cite{doerrDifficultiesNSGAIIManyObjective2025}. For an MPMOP with $M$ parties and $k_p$ objectives for party $p$, flattening yields a $K$-objective MaOP with $K=\sum_{p=1}^{M} k_p$. It therefore inherits the above many-objective difficulties while obscuring the structural fact that the consensus target is the common Pareto set rather than the full flattened front. This distinction is especially important in the bi-party bi-objective case, because such a problem is not simply a four-objective MOP. Its target is not the complete four-objective nondominated set, but the intersection of two party-wise bi-objective Pareto sets. Thus, runtime theories developed for full-front approximation do not adequately characterize algorithms that are explicitly driven toward common solutions. This mismatch motivates a dedicated runtime analysis tailored to MPMOPs.
	
	The theoretical benefits of crossover operators in overcoming mutation bottlenecks have been established in various evolutionary computation settings \cite{jansen2002analysis,jansen2005real,doerr2008crossover}. For multi-objective evolutionary optimization, Qian et al. \cite{qian2013analysis} showed that recombination can accelerate the filling of a Pareto front by recombining diverse solutions preserved in the population. Recent runtime analyses of crossover in multi-objective and many-objective optimization further sharpen this point. They show that crossover can induce substantial speedups when the landscape contains complementary partial structures that can be assembled by recombination, whereas mutation-only dynamics may have to rely on rare valley-crossing events \cite{doerrRuntimeAnalysisNSGAII2023a,doerrTheoreticalAnalysesMultiObjective2021,oprisMaopCrossoverAIJ2026}. This perspective is particularly relevant to multi-population MPMOP frameworks, because party-wise populations can preserve heterogeneous but potentially complementary information. This perspective raises a more specific theoretical question for MPMOPs concerning whether complementary structures maintained by different parties can be assembled by cross-party recombination into common Pareto-optimal solutions. Sun et al. \cite{sun2025runtime} provided an important first step by analyzing the runtime of SEMO and its variants on a benchmark MPMOP instance. However, several limitations remain. First, the benchmark landscape is relatively regular, with symmetric fitness structures for the two parties and no explicit gap region, so it does not isolate a gap-crossing bottleneck in consensus search. Second, the analysis is centered on SEMO rather than on mainstream MOEA frameworks such as NSGA-II. Third, the interaction mechanism mainly relies on alternating optimization and alternating dominance, so the role of cross-party recombination in constructing common Pareto-optimal solutions is not isolated. These limitations indicate that existing analyses still do not explain how cross-party recombination affects consensus search in MPMOPs. It also remains unclear whether analogous cross-party recombination effects can be formalized in structured combinatorial settings such as multi-party multi-objective minimum spanning tree problems, where feasible solutions must satisfy global connectivity constraints.
	
	The goal of this paper is to advance the theoretical understanding of EAs for MPMOPs by focusing on representative problem classes and widely adopted algorithmic frameworks. In particular, we investigate how multi-party interaction can be operationalized through cross-party recombination and shared evaluation of offspring across parties. To this end, we introduce a multi-population NSGA-II framework, called cross-party recombination NSGA-II (CPR-NSGA-II), in which each party maintains its own population and common-solution discovery is promoted by crossover across party populations together with shared offspring evaluation. The analysis is organized around two settings. The first is a pseudo-Boolean benchmark that isolates a cross-party gap-crossing bottleneck. The second is a bi-party bi-objective spanning-tree setting that examines common approximation in a graph-constrained combinatorial search space. The main contributions of this paper are summarized as follows.
	
	First, we introduce a pseudo-Boolean benchmark, MP-JCG (multi-party jump--count with gap). This benchmark combines structural ingredients from COCZ~\cite{doerr2025tight} and OJZJ~\cite{doerrTheoreticalAnalysesMultiObjective2021}, and is also motivated by the pseudo-Boolean benchmark tradition in runtime analysis of MOEAs~\cite{laumannsRunningTimeAnalysis2004a,qian2013analysis} and by recent pseudo-Boolean MPMOP analysis~\cite{sun2025runtime}, while introducing an explicit gap region that makes the discovery of common Pareto-optimal solutions non-trivial. On this problem, we analyze a payoff-guided mutation baseline inspired by the payoff-guided process in~\cite{sun2025runtime} and show that mutation-only search faces a gap-crossing bottleneck with expected fitness-evaluation complexity $\Theta(n^2)$. We then analyze the MP-JCG instantiation of CPR-NSGA-II and prove an $O(n\log n)$ expected fitness-evaluation bound for discovering the two common Pareto-optimal solutions under the analyzed setting. We further compare this result with the flattened four-objective formulation F-JCG, for which the expected full-front coverage complexity is shown to be $O(n^3+n^{k+1})$. These results show that cross-party recombination can replace a rare mutation event by directly assembling complementary prefix and suffix templates made available across party populations.
	
	Second, we extend the analysis to BPBOMST, the bi-party, two-objective-per-party specialization of the multi-party multi-objective minimum spanning tree problem. This setting is motivated by the bi-objective MST problem, a canonical combinatorial testbed for multi-objective evolutionary optimization, for which recombination-based Pareto-front coverage has been studied~\cite{qian2013analysis}, and NSGA-II runtime guarantees have recently been established~\cite{cerf2023}. It is also related to multi-agent network design problems in which several agents evaluate or compare the same graph through different edge preferences, rankings, or weights~\cite{li2023multiagentMSTCover,darmann2011finding,darmann2016condorcet}. The resulting problem is neither a standard bi-objective MST nor a flattened four-objective MST. In a bi-objective MST, one decision maker evaluates each tree through two objectives, yielding a two-dimensional Pareto front. In a flattened four-objective MST, the four party-wise objectives are concatenated, and the target becomes the full four-objective nondominated set. In BPBOMST, by contrast, each party has its own bi-objective Pareto set, and the search target is the set of spanning trees that are Pareto-optimal for both parties simultaneously.
	
	For BPBOMST, we develop a layered support-cover analysis. For each common Pareto objective vector, the symmetric average projection induces an auxiliary bi-objective MST. Suitable support representatives in the induced auxiliary spaces provide party-level factor-$2$ cover certificates, and lifting these certificates back to the original four objectives yields a $2\lambda$-common approximation mechanism with $\lambda\in[1,2]$. Thus, when the lifting loss satisfies $\lambda=1$, the analysis recovers a $2$-common approximation guarantee analogous to the classical bi-objective convex-cover argument, while the worst case $\lambda=2$ gives a $4$-common approximation cover. We further derive an instance-parameterized expected runtime bound for a representative-pool version of CPR-NSGA-II with edge-union recombination and uniform repair. The bound separates the local effort required to fill auxiliary convex-front segments, the part of this effort bypassed by CPR-good shortcuts, and the repair ambiguity introduced by edge-union recombination. This formulation explains when cross-party recombination accelerates common-cover construction by assembling complementary edge structures across parties. We also analyze the flattened counterpart and show, through a worst-case counting argument, that flattening can induce substantial nondominated-set growth and population-capacity pressure. In particular, a $K$-objective flattened formulation may require an archive or population of size $\Omega((n w_{\max})^{K-1})$ in a front-preserving elitist analysis, illustrating the dimension-driven burden introduced by flattening.
	
	Beyond the two concrete problem classes, the analysis identifies a common mechanism for the MPMOP settings studied in this paper. The key mechanism is cross-party complementarity, in which different decision makers may maintain or make available distinct but compatible fragments of a common solution, and cross-party recombination can assemble these fragments directly, replacing rare mutation-driven gap-crossing or multi-edge correction events. The results on MP-JCG and BPBOMST indicate that, in the analyzed settings, cross-party recombination acts not merely as a generic diversity-enhancing operator, but as a consensus-construction mechanism under explicit structural conditions.
	
	The remainder of the paper is organized as follows. Section~\ref{sec:preliminaries} introduces the MPMOP formulation, notation, and complexity convention used throughout the paper. Section~\ref{sec:mpjcg-analysis} studies MP-JCG and compares CPR-NSGA-II with a mutation-only baseline and with flattening. Section~\ref{sec:bpbomst-analysis} develops the common approximation and instance-parameterized runtime analysis for BPBOMST. Section~\ref{sec:experiments} provides supporting empirical observations. Section~\ref{sec:conclusion} discusses the mechanism, limitations, and implications of the analysis, and concludes the paper.
	
	\section{Preliminaries}
	\label{sec:preliminaries}
	
	In this section, we formalize multi-party multi-objective optimization, introduce the solution concepts used throughout the paper, and specify the NSGA-II abstraction and complexity convention adopted in the runtime analysis.
	
	\subsection{Problem Formulation and Solution Concepts}
	
	We first formalize MPMOPs and the associated party-wise and common Pareto concepts.
	
	\begin{definition}[Multi-party multi-objective optimization problem (MPMOP)] 
		\label{def:MPMOP}
		Following~\cite{liuEvolutionaryApproachMultiparty2020}, let $M \in \mathbb{N}$ denote the number of parties. For each party $p \in \{1,2,\ldots,M\}$, let $k_p\in\mathbb{N}$ denote the number of objectives of party $p$, and let $F_p : \mathcal{X} \to \mathbb{R}^{k_p}$ denote its objective vector, where
		\begin{equation}\label{eqt:MOPs}
			F_p(\mathbf{x}) = \left(f_{p,1}(\mathbf{x}), f_{p,2}(\mathbf{x}), \ldots, f_{p,k_p}(\mathbf{x})\right), \quad p = 1, 2, \ldots, M.
		\end{equation}
		An MPMOP is formulated as
		\begin{equation}\label{eqt:MPMOPs}
			\max_{\mathbf{x} \in \mathcal{X}} \mathcal{F}(\mathbf{x}) = \left(F_1(\mathbf{x}), F_2(\mathbf{x}), \ldots, F_M(\mathbf{x})\right),
		\end{equation}
		where the parties share the feasible decision space $\mathcal{X}$ but evaluate solutions through their own objective vectors.
	\end{definition}
	
	Unless stated otherwise, the definitions in this subsection use a maximization convention. In later sections, the pseudo-Boolean benchmark MP-JCG follows this convention directly. For BPBOMST in Section~\ref{sec:bpbomst-analysis}, we work with minimization; there, all dominance statements are interpreted with the inequality directions reversed. This convention switch is made explicit in the corresponding section whenever needed.
	
	For comparison, an MPMOP can be transformed into a standard multi-objective optimization problem by concatenating all party-wise objectives into a single objective vector. We formalize this transformation through the following flattening operation.
	
	\begin{definition}[Flattened multi-objective optimization problem (FMOP)]
		\label{def:flatten}
		Consider the MPMOP in Definition~\ref{def:MPMOP}. Let $K=\sum_{p=1}^{M} k_p$. The flattened objective vector is obtained by concatenating all party-wise objectives:
		\[
		\mathbf{f}(\mathbf{x})=
		\bigl(
		f_{1,1}(\mathbf{x}),\ldots,f_{1,k_1}(\mathbf{x}),
		f_{2,1}(\mathbf{x}),\ldots,f_{2,k_2}(\mathbf{x}),
		\ldots,
		f_{M,1}(\mathbf{x}),\ldots,f_{M,k_M}(\mathbf{x})
		\bigr)\in\mathbb{R}^K.
		\]
		The corresponding FMOP is
		\[
		\max_{\mathbf{x}\in\mathcal{X}} \mathbf{f}(\mathbf{x}).
		\]
	\end{definition}
	
	The FMOP will later serve as a comparison baseline. In the analysis sections below, we show that flattening may enlarge the search target and weaken selection pressure in ways that are unfavorable for consensus discovery.
	
	In the MPMOP formulation, however, solution quality is evaluated separately for each party, which motivates the party-wise dominance concepts introduced below. We therefore define Pareto domination separately for each party.
	
	\begin{definition}[Party-wise Pareto domination]
		\label{def:Domination}
		Given two solutions $\mathbf{x}, \mathbf{x}' \in \mathcal{X}$ and a party $p$, we say that $\mathbf{x}$ weakly dominates $\mathbf{x}'$ with respect to party $p$, denoted $\mathbf{x} \succeq_p \mathbf{x}'$, if
		\[
		f_{p,j}(\mathbf{x}) \ge f_{p,j}(\mathbf{x}'), \quad j \in \{1,\ldots,k_p\}.
		\]
		If, additionally, the inequality is strict for at least one objective, then $\mathbf{x}$ dominates $\mathbf{x}'$ with respect to party $p$, denoted $\mathbf{x} \succ_p \mathbf{x}'$.
	\end{definition}
	
	The Pareto set induced by party $p$ is defined as
	\[
	\mathcal{PS}_p=\bigl\{\mathbf{x}\in\mathcal{X}\mid \nexists\,\mathbf{x}'\in\mathcal{X}\text{ such that }\mathbf{x}'\succ_p \mathbf{x}\bigr\},
	\]
	and the corresponding Pareto front is $\mathcal{PF}_p=F_p(\mathcal{PS}_p)$.
	
	In multi-party settings, a central solution concept of interest is the possibly empty set of solutions that are Pareto-optimal for all parties simultaneously.
	
	\begin{definition}[Common Pareto set]
		\label{def:CommonPS}
		The \emph{common Pareto set} of an MPMOP is defined as
		\begin{equation}
			\label{eq:commonPS}
			\mathcal{PS}_{\mathrm{com}} = \bigcap_{p=1}^M \mathcal{PS}_p.
		\end{equation}
		A solution $\mathbf{x} \in \mathcal{X}$ is called \emph{common Pareto-optimal} if $\mathbf{x} \in \mathcal{PS}_{\mathrm{com}}$.
	\end{definition}
	
	In the main runtime analyses of this paper, we assume that the common Pareto set is non-empty. This non-emptiness assumption is imposed only for the analyzed instances and is not part of the general MPMOP definition. For some arguments, it is also convenient to use a dominance relation that jointly accounts for all parties.
	
	\begin{definition}[Multi-party Pareto domination]
		\label{def:Multi-Party_Domination}
		Given two solutions $\mathbf{x}, \mathbf{x}' \in \mathcal{X}$, we say that $\mathbf{x}$ multi-party weakly dominates $\mathbf{x}'$, denoted $\mathbf{x} \succeq_{\mathrm{MP}} \mathbf{x}'$, if
		\[
		\mathbf{x} \succeq_p \mathbf{x}' \quad \text{for all } p\in\{1,\ldots,M\}.
		\]
		If, additionally, there exists at least one party $p_0$ such that $\mathbf{x} \succ_{p_0} \mathbf{x}'$, then $\mathbf{x}$ multi-party dominates $\mathbf{x}'$, denoted $\mathbf{x} \succ_{\mathrm{MP}} \mathbf{x}'$.
	\end{definition}
	
	\subsection{NSGA-II and Population Update}
	\label{subsec:nsgaii}
	
	NSGA-II~\cite{debFastElitistMultiobjective2002} is a widely used multi-objective evolutionary algorithm based on non-dominated sorting and crowding distance. In this paper, the CPR-NSGA-II variants analyzed later use the NSGA-II population-update mechanism within party-wise populations. We therefore recall the basic NSGA-II procedure and its population update. Because MP-JCG and BPBOMST use different encodings and variation operators, the offspring-generation step is specified in the corresponding analysis sections.
	
	\begin{algorithm}[H]
		\caption{Basic NSGA-II procedure}
		\label{alg:NSGA-II}
		\begin{algorithmic}[1]
			\REQUIRE Population size $N$, search space $\mathcal{X}$, objective functions
			\ENSURE Final population $P$
			\STATE $P \leftarrow$ initialize $N$ feasible solutions from $\mathcal{X}$
			\WHILE{stopping criterion is not met}
			\STATE $P' \leftarrow$ generate $N$ offspring from $P$ using the offspring-generation operator specified for the current problem setting
			\STATE $P \leftarrow \textsc{PopulationUpdate}(P \cup P', N)$
			\ENDWHILE
			\RETURN $P$
		\end{algorithmic}
	\end{algorithm}
	
	Algorithm~\ref{alg:NSGA-II} is used only to recall the standard NSGA-II update structure. The concrete CPR-NSGA-II variants in Sections~\ref{sec:mpjcg-analysis} and~\ref{sec:bpbomst-analysis} add party-wise populations, cross-party recombination, shared offspring evaluation, and, depending on the setting, archives, immigrants, or repair operators.
	
	\begin{algorithm}[H]
		\caption{Population update}
		\label{alg:NSGA-II-Update}
		\begin{algorithmic}[1]
			\REQUIRE Combined population $Q$, target size $N$
			\ENSURE Updated population $P$ with $|P|\le N$
			\STATE $\mathcal{R} = \{R_1, R_2, \ldots, R_v\} \leftarrow$ sort $Q$ into non-dominated fronts
			\STATE $P \leftarrow \emptyset, i \leftarrow 1$
			\WHILE{$i \le v$ and $|P|+|R_i| \le N$}
			\STATE $P \leftarrow P \cup R_i$
			\STATE $i \leftarrow i + 1$
			\ENDWHILE
			\IF{$|P| < N$ and $i \le v$}
			\STATE compute crowding distance for individuals in $R_i$
			\STATE sort $R_i$ based on crowding distance in descending order
			\STATE $P \leftarrow P \cup \{ \text{first } N - |P| \text{ solutions from } R_i \}$
			\ENDIF
			\RETURN $P$
		\end{algorithmic}
	\end{algorithm}
	
	NSGA-II has been studied theoretically on several benchmark and combinatorial multi-objective problems, including settings in which the population size is sufficient to cover the Pareto front and settings involving multimodal fitness landscapes or many-objective search~\cite{zheng2022,cerf2023,doerr2023,bian2022,zheng2024}. These results provide methodological background for the runtime analyses developed later in the paper.
	
	\subsection{Notation Summary}
	
	For convenience and clarity, Table~\ref{tab:notation} summarizes the key mathematical notation used throughout the paper. Some symbols are section-specific; when a symbol is overloaded later, the local definition in that section takes precedence.
	
	\begin{table}[h]
		\caption{Summary of key notation}
		\label{tab:notation}
		\centering
		\resizebox{\columnwidth}{!}{
			\begin{tabular}{ll|ll}
				\toprule
				\textbf{Symbol} & \textbf{Meaning} & \textbf{Symbol} & \textbf{Meaning} \\
				\midrule
				$M$ & Number of parties & $K$ & Flattened objective dimension \\
				$\mathcal X$ & Feasible decision space & $N$ & Population-size parameter \\
				$F_p$ & Objective vector of party $p$ & $\mathcal{PS}_{\mathrm{com}}$ & Common Pareto set \\
				$\mathcal H$ & Gap set in MP-JCG & $i(\mathbf{x}), b(\mathbf{x})$ & Prefix-one / suffix-zero counts in MP-JCG \\
				$\Psi(\mathbf{x})$ & Payoff potential in MP-JCG & $W$ & Upper bound on MST objective values \\
				$Y(T)$ & Joint objective vector in BPBOMST & $A_y^{\mathrm{avg}}$ & Average-projection auxiliary objective \\
				$F_y^A$ & Auxiliary Pareto front & $C_A$ & Total auxiliary convex-front filling size \\
				$C_{\min}^A$ & Smallest nonempty auxiliary segment size & $N_{\mathrm{CPR}}$ & Number of CPR-good segments \\
				$G_{\mathrm{CPR}}$ & Local-filling gain from CPR shortcuts & $\Omega_{\mathrm{CPR}}$ & Total edge-union repair ambiguity \\
				$C_{\mathrm{pw}}$ & Party-wise provider-front size & $p_g$ & CPR probability \\
				\bottomrule
			\end{tabular}
		}
	\end{table}
	
	\subsection{Complexity Measurement Convention}
	
	To avoid ambiguity across sections, we use a fixed complexity-reporting convention throughout the paper.
	\begin{itemize}
		\item We report fitness-evaluation (FE) complexity as the main complexity quantity.
		\item Generation bounds are used as intermediate proof steps and are converted to FE complexity explicitly.
		\item For an algorithm with expected generation complexity $\mathbb{E}[G]$, initialization budget $B_{\mathrm{init}}$, and per-generation evaluation budget $B_{\mathrm{gen}}$, we use
		\[
		\mathbb{E}[T_{\mathrm{FE}}]=B_{\mathrm{init}}+B_{\mathrm{gen}}\,\mathbb{E}[G].
		\]
	\end{itemize}
	
	In our settings, $B_{\mathrm{gen}}=\Theta(N)$, where $N$ denotes the relevant population-size or offspring-budget parameter in the current algorithmic setting. Equivalently, $B_{\mathrm{gen}}$ is proportional to the number of newly evaluated candidates per generation, including offspring and immigrants when present, up to constant factors independent of problem size. Initialization evaluations are counted in $B_{\mathrm{init}}$.
	
	\noindent\textbf{Remark on notation.} In the concrete algorithmic sections, the same letter $N$ is sometimes used for a per-population size parameter. Whenever this occurs, the FE conversion is stated explicitly in the corresponding theorem or corollary to avoid ambiguity.

	\section{Runtime Analysis on MP-JCG}
	\label{sec:mpjcg-analysis}
	
	In the existing theoretical literature on MPMOPs over pseudo-Boolean domains, the BPAOAZ instance proposed by Sun et al.~\cite{sun2025runtime} provides a useful starting point for runtime analysis and is, to the best of our knowledge, one of the few pseudo-Boolean benchmarks explicitly designed for this setting. Owing to its symmetric structure derived from two COCZ components, BPAOAZ admits a comparatively regular fitness landscape, which facilitates theoretical treatment. However, this symmetric construction does not explicitly model separated attraction basins induced by low-fitness regions.
	
	In many multi-party decision scenarios, structural constraints may create low-quality valleys between locally attractive regions, thereby increasing the difficulty of global coordination. To capture this phenomenon while preserving analytical tractability, we introduce a new pseudo-Boolean benchmark termed multi-party jump-count with gap (MP-JCG). The construction combines structural elements of OJZJ~\cite{doerrTheoreticalAnalysesMultiObjective2021} and the gap-based COCZ construction~\cite{doerr2025tight}, and introduces an explicit gap region that makes the discovery of common Pareto-optimal solutions non-trivial. This construction isolates a consensus-search bottleneck in which party-wise structures are individually reachable, whereas their common assembly requires coordinating complementary components across parties.
	
	We first define MP-JCG and characterize the Pareto sets of the two parties and their common Pareto set. These structural results are independent of the algorithmic assumptions used later. We then analyze the corresponding runtime behavior of evolutionary algorithms on this benchmark. Throughout this section, runtime is measured in fitness evaluations, with generation bounds converted according to the convention in Sec.~\ref{sec:preliminaries}.
	
	\subsection{Multi-Party Jump-Count with Gap}
	
	We define MP-JCG as a bi-party bi-objective pseudo-Boolean benchmark. The two parties share the decision space $\{0,1\}^n$ and are constructed from OJZJ and a gap-based COCZ component, respectively; for brevity, we call the second component G-COCZ. The problem is presented in maximization form to match the notation in Sec.~\ref{sec:preliminaries}. The parameter \(k\) specifies the size of the jump/gap structure; throughout this section, unless stated otherwise, we assume \(2\le k\le \lfloor n/2\rfloor\).
	
	\begin{definition}[Multi-Party Jump-Count with Gap (MP-JCG)]
		\label{def:MP-JCG}
		Let $n\ge4$ and $k \in \{2,\dots,\lfloor n/2\rfloor\}$. The MP-JCG is defined as a bi-party bi-objective pseudo-Boolean function
		\[
		\mathrm{MP\mbox{-}JCG} : \{0,1\}^n \to \mathbb{N}_0^2 \times \mathbb{N}_0^2, \quad
		\mathbf{x} = (x_1, \dots, x_n) \mapsto (\mathrm{OJZJ}(\mathbf{x}), \mathrm{G\mbox{-}COCZ}(\mathbf{x})),
		\]
		where $\mathbb{N}_0=\{0,1,2,\ldots\}$.
		
		For any interval $a \le b$, define
		\[
		|\mathbf{x}_a^b|_1 = \sum_{i=a}^b x_i, \qquad
		|\mathbf{x}_a^b|_0 = \sum_{i=a}^b (1-x_i), \qquad
		|\mathbf{x}|_1 = |\mathbf{x}_1^n|_1, \qquad
		|\mathbf{x}|_0 = |\mathbf{x}_1^n|_0 .
		\]
		
		\noindent\textbf{Party 1.} $\mathrm{OJZJ}(\mathbf{x}) = (f_{1,1}(\mathbf{x}), f_{1,2}(\mathbf{x}))$ is defined by
		\[
		f_{1,1}(\mathbf{x}) =
		\begin{cases}
			k + |\mathbf{x}|_1, & \text{if } |\mathbf{x}|_1 \le n-k \text{ or } \mathbf{x} = 1^n,\\
			n - |\mathbf{x}|_1, & \text{otherwise},
		\end{cases}
		\qquad
		f_{1,2}(\mathbf{x}) =
		\begin{cases}
			k + |\mathbf{x}|_0, & \text{if } |\mathbf{x}|_0 \le n-k \text{ or } \mathbf{x} = 0^n,\\
			n - |\mathbf{x}|_0, & \text{otherwise}.
		\end{cases}
		\]
		
		\noindent\textbf{Party 2.} $\mathrm{G\mbox{-}COCZ}(\mathbf{x}) = (f_{2,1}(\mathbf{x}), f_{2,2}(\mathbf{x}))$ is defined by
		\[
		f_{2,1}(\mathbf{x}) =
		\begin{cases}
			0, & \mathbf{x} \in \mathcal{H},\\
			|\mathbf{x}|_1, & \text{otherwise},
		\end{cases}
		\qquad
		f_{2,2}(\mathbf{x}) =
		\begin{cases}
			0, & \mathbf{x} \in \mathcal{H},\\
			|\mathbf{x}_1^{\,n-k}|_1 + |\mathbf{x}_{n-k+1}^{\,n}|_0, & \text{otherwise},
		\end{cases}
		\]
		where
		\[
		\mathcal{H} = \left\{\mathbf{x} \in \{0,1\}^n : |\mathbf{x}_1^{\,n-k}|_1 = n-k \ \land\ |\mathbf{x}_{n-k+1}^{\,n}|_0 = 1 \right\}.
		\]
	\end{definition}
	
	Intuitively, the gap set $\mathcal{H}$ represents an unfavorable region of the fitness landscape that blocks a local single-bit route to one of the target common Pareto-optimal solutions. A solution falling into $\mathcal{H}$ is assigned zero values on both Party~2 objectives. The construction captures a simple incompatibility pattern---a solution may be close to the desired structure in Hamming distance, but becomes unacceptable to one party unless a coordinated correction is made.
	
	Fig.~\ref{fig:pareto_fronts_n8} visualizes the objective spaces of both parties for $n=8$ and $k=3$. 
	
	\begin{figure}[htbp]
		\centering
		\includegraphics[width=1.0\linewidth]{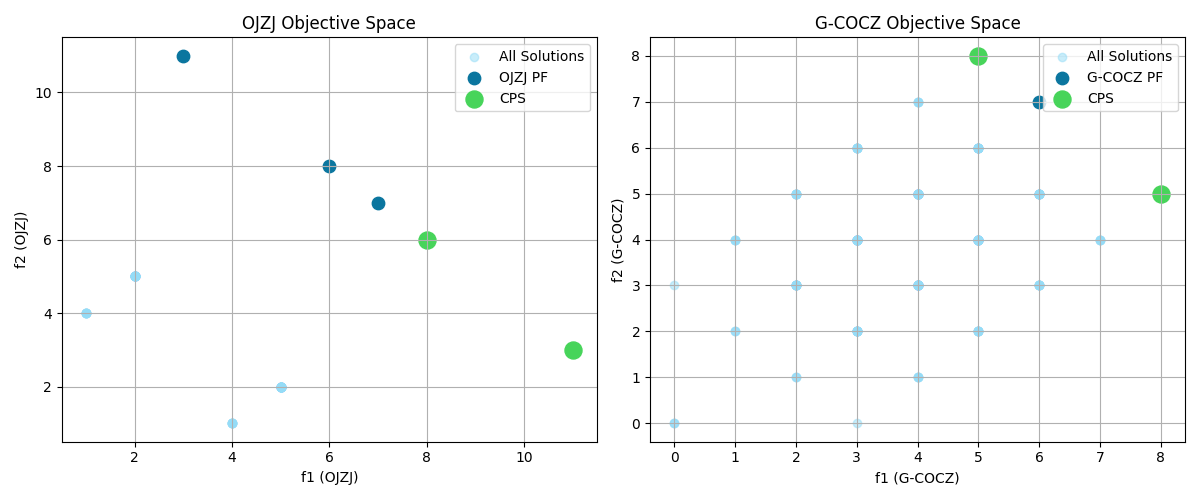}
		\caption{Objective spaces of Party~1 (OJZJ, left) and Party~2 (G-COCZ, right) for $n=8,k=3$. Dark blue points indicate party-wise Pareto-optimal objective vectors, and green points mark the objective vectors of common Pareto-optimal solutions. The gap layer in Party~2 removes the full-prefix points with exactly one suffix zero, creating a local obstruction for single-bit improvement.}
		\label{fig:pareto_fronts_n8}
	\end{figure}
	
	We next characterize the Pareto sets of the two parties and their common Pareto set.
	
	\begin{proposition}
		\label{prop:ps1-mpjcg}
		For Party~1 of MP-JCG,
		\[
		\mathcal{PS}_1 = \{0^n,1^n\} \cup \left\{\mathbf{x}\in\{0,1\}^n : |\mathbf{x}|_1 \in \{k,\ldots,n-k\}\right\}.
		\]
	\end{proposition}
	
	\begin{proof}
		Let $t=|\mathbf{x}|_1$. The objective vector of Party~1 depends only on $t$ and takes the following form:
		\[
		(f_{1,1}(\mathbf{x}),f_{1,2}(\mathbf{x}))=
		\begin{cases}
			(k,n+k), & t=0,\\
			(k+t,t), & 1\le t\le k-1,\\
			(k+t,k+n-t), & k\le t\le n-k,\\
			(n-t,k+n-t), & n-k+1\le t\le n-1,\\
			(n+k,k), & t=n.
		\end{cases}
		\]
		
		For $t\in\{k,\ldots,n-k\}$, increasing $t$ improves $f_{1,1}$ and decreases $f_{1,2}$, so these objective vectors are mutually non-dominated. The extreme points $0^n$ and $1^n$ are also non-dominated, since each maximizes one objective.
		
		It remains to show that all other points are dominated. If $1\le t\le k-1$, then any point $\mathbf{y}$ with $|\mathbf{y}|_1=k$ satisfies
		\[
		(f_{1,1}(\mathbf{y}),f_{1,2}(\mathbf{y}))=(2k,n)\succ (k+t,t)=(f_{1,1}(\mathbf{x}),f_{1,2}(\mathbf{x})).
		\]
		If $n-k+1\le t\le n-1$, then any point $\mathbf{z}$ with $|\mathbf{z}|_1=n-k$ satisfies
		\[
		(f_{1,1}(\mathbf{z}),f_{1,2}(\mathbf{z}))=(n,2k)\succ (n-t,k+n-t)=(f_{1,1}(\mathbf{x}),f_{1,2}(\mathbf{x})).
		\]
		Hence no point outside the stated set is Pareto-optimal.
	\end{proof}
	
	\begin{proposition}
		\label{prop:ps2-mpjcg}
		For Party~2 of MP-JCG,
		\[
		\mathcal{PS}_2 =
		\left\{
		\mathbf{x}\in\{0,1\}^n :
		|\mathbf{x}_1^{\,n-k}|_1 = n-k,\ 
		|\mathbf{x}_{n-k+1}^{\,n}|_0 \neq 1
		\right\}.
		\]
	\end{proposition}
	
	\begin{proof}
		Let
		\[
		a=|\mathbf{x}_1^{\,n-k}|_0,\qquad b=|\mathbf{x}_{n-k+1}^{\,n}|_0.
		\]
		If $\mathbf{x}\notin\mathcal{H}$, then
		\[
		f_{2,1}(\mathbf{x}) = n-a-b,\qquad
		f_{2,2}(\mathbf{x}) = (n-k-a)+b.
		\]
		
		We first show that any point with $a>0$ is dominated. If $a>0$ and $b\neq 1$, define $\mathbf{y}$ as the point with all first $n-k$ bits equal to $1$ and exactly $b$ zero-bits in the last $k$ positions. Then $\mathbf{y}\notin\mathcal{H}$ and
		\[
		f_{2,1}(\mathbf{y})=n-b \ge n-a-b=f_{2,1}(\mathbf{x}),
		\]
		\[
		f_{2,2}(\mathbf{y})=n-k+b \ge n-k-a+b=f_{2,2}(\mathbf{x}),
		\]
		with at least one strict inequality since $a>0$. Thus $\mathbf{y}\succ_2 \mathbf{x}$.
		
		If $a>0$ and $b=1$, then $\mathbf{x}\notin\mathcal{H}$ automatically, and
		\[
		(f_{2,1}(1^n),f_{2,2}(1^n))=(n,n-k)
		\]
		dominates
		\[
		(f_{2,1}(\mathbf{x}),f_{2,2}(\mathbf{x}))=(n-a-1,n-k-a+1).
		\]
		
		Therefore, every Pareto-optimal point must satisfy $a=0$, that is, all first $n-k$ bits are equal to $1$. Among such points, the case $b=1$ is exactly the gap set $\mathcal H$. Every point in $\mathcal H$ has Party~2 objective vector $(0,0)$ and is dominated by $1^n$, whose Party~2 objective vector is $(n,n-k)$. For the remaining points with $a=0$ and $b\neq 1$, the objective vector is
		\[
		(f_{2,1}(\mathbf{x}),f_{2,2}(\mathbf{x}))=(n-b,n-k+b).
		\]
		As $b$ increases, the first objective decreases while the second increases, so these points are mutually non-dominated. Hence the stated set is exactly $\mathcal{PS}_2$.
	\end{proof}
	
	\begin{corollary}
		\label{cor:cps-mpjcg}
		The common Pareto set of MP-JCG is
		\[
		\mathcal{PS}_{\mathrm{com}}=\mathcal{PS}_1\cap\mathcal{PS}_2=\{1^{n-k}0^k,\,1^n\}.
		\]
	\end{corollary}
	
	\begin{proof}
		By Proposition~\ref{prop:ps2-mpjcg}, every point in $\mathcal{PS}_2$ has all first $n-k$ bits equal to $1$, and the last $k$ bits contain $b\neq 1$ zeros for some $b\in\{0,\dots,k\}$. Hence such a point has Hamming weight $n-b$.
		
		To also belong to $\mathcal{PS}_1$, Proposition~\ref{prop:ps1-mpjcg} implies that either $n-b=n$, which gives $b=0$ and thus $\mathbf{x}=1^n$, or $n-b\in\{k,\ldots,n-k\}$. Since $0\le b\le k$ and $n\ge 2k$, the latter is possible only when $b=k$. Because $k\ge2$, this value is not excluded by the condition $b\neq1$, and it gives $\mathbf{x}=1^{n-k}0^k$.
		
		Therefore,
		\[
		\mathcal{PS}_{\mathrm{com}}=\{1^{n-k}0^k,\,1^n\}.
		\]
	\end{proof}
	
	Corollary~\ref{cor:cps-mpjcg} shows the structural role of the benchmark. Party~2 exposes the full-prefix structure shared by both common solutions, whereas the two common solutions differ only in the constant-size suffix block. This separation between a long prefix structure and a short suffix block is the source of the recombination advantage analyzed below.
	
	\subsection{Runtime Analysis of CPR-NSGA-II on MP-JCG}
	\label{subsec:cpr-mpjcg}
	
	We analyze a bi-population NSGA-II variant, called cross-party recombination NSGA-II (CPR-NSGA-II), on MP-JCG.
	The algorithm maintains one population for each party, preserves party-specific search information through party-wise environmental selection, and allows offspring construction from parents drawn across populations.
	The target event is the discovery of the two common Pareto-optimal solutions
	\[
	\mathcal{PS}_{\mathrm{com}}=\{1^{n-k}0^k,\,1^n\}.
	\]
	An auxiliary common archive is maintained to record evaluated candidates under the multi-party dominance relation. It is not used as an additional parent population in the analysis and is monotone on common Pareto-optimal solutions.
	
	For analytical transparency, we study an analysis-oriented instantiation of CPR-NSGA-II with a constant population-size parameter, rank-based binary tournament without crowding-distance tie-breaking, and uniform truncation within an overflowing last accepted front.
	Each party also receives one uniformly random immigrant per generation. This step makes constant-length suffix building blocks available with constant probability. The asymptotic speedup, however, comes from cross-party recombination rather than from random immigration itself.
	The asymptotic speedup proved below follows from cross-party recombination once complementary prefix and suffix structures are simultaneously available across the two party populations.
	
	Recall that $k\ge2$ by Definition~\ref{def:MP-JCG}. Throughout this subsection, we assume that $k=O(1)$.
	We further assume that the population-size parameter $N$ is a fixed constant satisfying
	\[
	N \ge k+1.
	\]
	We also assume that $p_g=\Omega(1)$ and that $p_c$ is bounded away from both $0$ and $1$.
	Since $k=O(1)$, the population-size parameter is constant.
	
	For the runtime proof, we use duplicate removal in objective space before each party-wise population update. For a candidate set $Q$ and party $i$, let $\operatorname{Rep}_i(Q)$ contain one arbitrary search point for each distinct vector in $F_i(Q)$. The update in Algorithm~\ref{alg:cpr-nsgaii-mpjcg} is written as $\textsc{PopulationUpdate}_i(\operatorname{Rep}_i(Q),N)$. This means that duplicate removal in objective space is applied before the population-update procedure from Sec.~\ref{subsec:nsgaii}. The subscript $i$ indicates that nondominated sorting is computed under the party-wise objective vector $F_i$. In this MP-JCG analysis, an overflowing last accepted front is truncated uniformly, as stated above. This convention removes only objective-vector duplicates and does not change the attainable objective vectors or the common Pareto-optimal solutions. The temporary parent pool $\widehat P_t^i$ still contains the newly sampled immigrant in the generation in which it is sampled.
	
	Parent selection is performed by rank-based binary tournament. Two individuals are sampled uniformly with replacement from the current parent pool, and the one with lower nondomination rank is selected. Ties are broken uniformly at random.
	Environmental selection inserts complete nondominated fronts greedily, and if the next front exceeds the remaining capacity, the remaining slots are filled by uniform random sampling from that front.
	Variation is given by one-point crossover followed by standard bit mutation with mutation rate $1/n$.
	
	For two parents $\mathbf{u},\mathbf{v}\in\{0,1\}^n$, one-point crossover selects a cut position $c$ uniformly from $\{1,\dots,n-1\}$ and returns
	\[
	\textsc{OnePointCrossover}(\mathbf{u},\mathbf{v})
	=
	(\mathbf{u}_1^c,\mathbf{v}_{c+1}^n),
	\]
	that is, the offspring inherits its prefix from the first parent and its suffix from the second parent.
	
	The algorithm is given in Algorithm~\ref{alg:cpr-nsgaii-mpjcg}.
	The common archive is updated by retaining the multi-party nondominated subset of all evaluated candidates seen so far, as described in Algorithm~\ref{alg:update-common-archive-mpjcg}.
	
	\begin{algorithm}[h]
		\caption{CPR-NSGA-II on MP-JCG}
		\label{alg:cpr-nsgaii-mpjcg}
		\begin{algorithmic}[1]
			\REQUIRE Population-size parameter $N$, bi-party objective vectors $F_1,F_2$, inter-party mating probability $p_g\in(0,1)$, crossover probability $p_c\in(0,1)$
			\ENSURE Common archive $A_{\mathrm{com}}$
			\STATE Sample two initial candidate sets $R_0^1$ and $R_0^2$ uniformly at random with $|R_0^1|=|R_0^2|=N$
			\STATE Evaluate all search points in $R_0^1\cup R_0^2$ under both $F_1$ and $F_2$
			\STATE $P_0^1 \leftarrow \textsc{PopulationUpdate}_1(\operatorname{Rep}_1(R_0^1),N)$
			\STATE $P_0^2 \leftarrow \textsc{PopulationUpdate}_2(\operatorname{Rep}_2(R_0^2),N)$
			\STATE $A_{\mathrm{com}} \leftarrow \textsc{UpdateCommonArchive}(R_0^1\cup R_0^2)$
			\FOR{$t=0,1,2,\dots$ until the termination criterion is met}
			\STATE Sample one random immigrant $\mathbf{r}_t^1\sim\mathrm{Unif}(\{0,1\}^n)$ and one random immigrant $\mathbf{r}_t^2\sim\mathrm{Unif}(\{0,1\}^n)$
			\STATE Evaluate $\mathbf{r}_t^1$ and $\mathbf{r}_t^2$ under both $F_1$ and $F_2$
			\STATE $\widehat P_t^1 \leftarrow P_t^1 \cup \{\mathbf{r}_t^1\}$, \qquad $\widehat P_t^2 \leftarrow P_t^2 \cup \{\mathbf{r}_t^2\}$
			\STATE $O_t^1 \leftarrow \emptyset$, \qquad $O_t^2 \leftarrow \emptyset$
			\FOR{each party $i\in\{1,2\}$}
			\FOR{$r=1$ to $N$}
			\STATE Select a primary parent $\mathbf{a}$ from $\widehat P_t^i$ by rank-based binary tournament under party $i$
			\STATE Select a secondary parent $\mathbf{b}$ from $\widehat P_t^{3-i}$ with probability $p_g$, and from $\widehat P_t^i$ otherwise, using rank-based binary tournament under the corresponding party
			\STATE Apply one-point crossover to $(\mathbf{a},\mathbf{b})$ with probability $p_c$; otherwise let the intermediate offspring equal $\mathbf{a}$
			\STATE Obtain $\mathbf{x}'$ by standard bit mutation with mutation rate $1/n$
			\STATE Evaluate $\mathbf{x}'$ under both $F_1$ and $F_2$
			\STATE $O_t^i \leftarrow O_t^i \cup \{\mathbf{x}'\}$
			\ENDFOR
			\ENDFOR
			\STATE $P_{t+1}^1 \leftarrow \textsc{PopulationUpdate}_1\!\left(\operatorname{Rep}_1(\widehat P_t^1\cup O_t^1),N\right)$
			\STATE $P_{t+1}^2 \leftarrow \textsc{PopulationUpdate}_2\!\left(\operatorname{Rep}_2(\widehat P_t^2\cup O_t^2),N\right)$
			\STATE $Q_t \leftarrow A_{\mathrm{com}} \cup \widehat P_t^1 \cup \widehat P_t^2 \cup O_t^1 \cup O_t^2$
			\STATE $A_{\mathrm{com}} \leftarrow \textsc{UpdateCommonArchive}(Q_t)$
			\ENDFOR
			\STATE \textbf{return} $A_{\mathrm{com}}$
		\end{algorithmic}
	\end{algorithm}
	
	\begin{algorithm}[h]
		\caption{\textsc{UpdateCommonArchive}}
		\label{alg:update-common-archive-mpjcg}
		\begin{algorithmic}[1]
			\REQUIRE Candidate set $Q$
			\ENSURE Updated common archive $A_{\mathrm{com}}$
			\STATE Remove duplicate search points from $Q$
			\STATE $A_{\mathrm{com}} \leftarrow \{\mathbf{x}\in Q : \nexists\, \mathbf{y}\in Q \text{ such that } \mathbf{y}\succ_{\mathrm{MP}} \mathbf{x}\}$
			\STATE \textbf{return} $A_{\mathrm{com}}$
		\end{algorithmic}
	\end{algorithm}
	
	Let $A_{\mathrm{com},t}$ denote the common archive after generation $t$.
	By Corollary~\ref{cor:cps-mpjcg},
	\[
	\mathcal{PS}_{\mathrm{com}}=\{1^{n-k}0^k,\,1^n\}.
	\]
	Accordingly, we define
	\[
	T_{\mathrm{com}}
	:=
	\inf\{t\ge 0 : \{1^{n-k}0^k,1^n\}\subseteq A_{\mathrm{com},t}\}.
	\]
	
	The mechanism analyzed below is illustrated in Fig.~\ref{fig:sec3-mpjcg-schematic}. Party~2 first provides a search point whose prefix is close to $1^{n-k}$ but not necessarily complete. Party~1 can supply a constant-length suffix block of type $0^k$ or $1^k$. A one-point crossover at the boundary then combines the Party~2 prefix with the Party~1 suffix, after which only a constant number of prefix corrections remain.
	
	\begin{figure}[ht]
		\centering
		\includegraphics[width=1\linewidth]{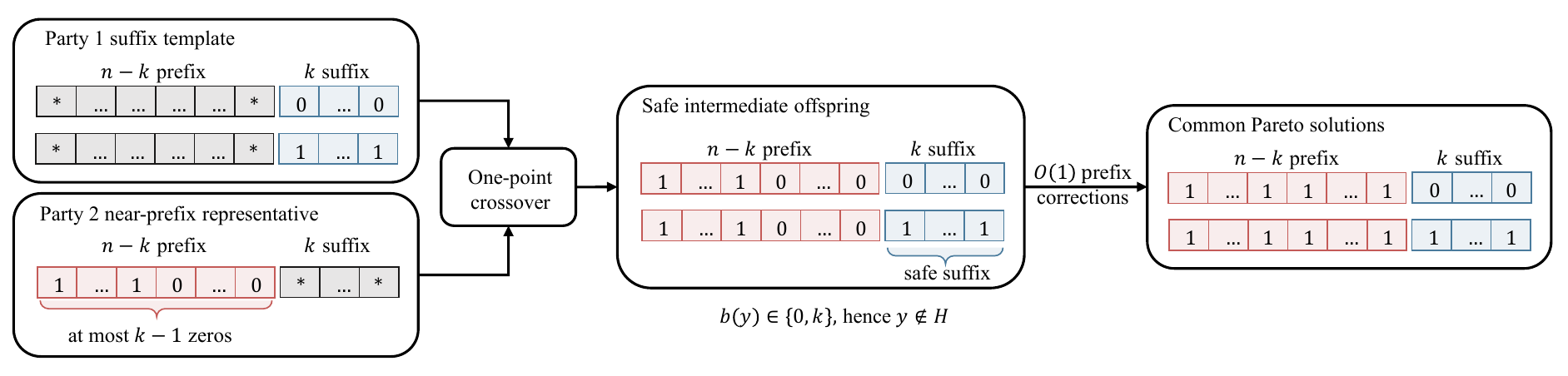}
		\caption{Boundary-crossover mechanism on MP-JCG. Party~2 provides a
			near-prefix representative whose prefix contains at most \(k-1\) remaining
			zero-bits, while Party~1 supplies a suffix template \(0^k\) or \(1^k\).
			A one-point crossover at the boundary \(c=n-k\) combines the Party~2 prefix
			with the Party~1 suffix, producing a safe intermediate offspring with
			\(b\in\{0,k\}\) and hence avoiding the gap layer \(\mathcal{H}\). After at most
			\(k-1\) prefix corrections, the search reaches one of the two common
			Pareto-optimal solutions, \(1^{n-k}0^k\) or \(1^n\).}
		\label{fig:sec3-mpjcg-schematic}
	\end{figure}
	
	For the proof, we track each search point through its prefix and suffix coordinates.
	For every $\mathbf{x}\in\{0,1\}^n$, let
	\[
	i(\mathbf{x}) := |\mathbf{x}_1^{\,n-k}|_1,
	\qquad
	b(\mathbf{x}) := |\mathbf{x}_{n-k+1}^{\,n}|_0.
	\]
	Thus, $i(\mathbf{x})$ is the number of 1-bits in the prefix of length $n-k$, and $b(\mathbf{x})$ is the number of 0-bits in the suffix of length $k$.
	
	We define
	\[
	\mathcal D_{\mathrm{pre}}
	:=
	\left\{
	\mathbf{x}\in\{0,1\}^n :
	\mathbf{x}\notin\mathcal H,\ 
	i(\mathbf{x})\ge n-2k+1
	\right\},
	\]
	and the corresponding first hitting time by
	\[
	T_{\mathrm{pre}}
	:=
	\inf\{t\ge 0 : P_t^2\cap \mathcal D_{\mathrm{pre}} \neq \emptyset\}.
	\]
	Every point in $\mathcal D_{\mathrm{pre}}$ has at most $k-1$ missing one-bits in its prefix, because the prefix length is $n-k$.
	
	For Party~1, define the two suffix-block classes
	\[
	\mathcal T_0
	:=
	\left\{
	\mathbf{x}\in\{0,1\}^n :
	\mathbf{x}_1^k = 1^k,\ 
	\mathbf{x}_{n-k+1}^{\,n}=0^k
	\right\},
	\]
	and
	\[
	\mathcal T_1
	:=
	\left\{
	\mathbf{x}\in\{0,1\}^n :
	\mathbf{x}_1^k = 0^k,\ 
	\mathbf{x}_{n-k+1}^{\,n}=1^k
	\right\}.
	\]
	Both sets are subsets of $\mathcal{PS}_1$.
	A uniformly random Party~1 immigrant belongs to each of $\mathcal T_0$ and $\mathcal T_1$ with probability $2^{-2k}$, and hence belongs to $\mathcal T_0\cup\mathcal T_1$ with probability $2^{1-2k}=\Omega(1)$ because $k=O(1)$.
	
	\begin{lemma}[Party-2 objective-vector representatives]
		\label{lem:party2-representatives}
		Let $S$ be any set containing at least one non-gap point, and consider its Party~2 objective-vector representatives.
		For every non-gap point $\mathbf{x}\notin\mathcal H$,
		\[
		F_2(\mathbf{x})
		=
		\bigl(i(\mathbf{x})+k-b(\mathbf{x}),\; i(\mathbf{x})+b(\mathbf{x})\bigr).
		\]
		Within any fixed suffix class $b$, a larger prefix value $i$ strictly dominates a smaller prefix value under Party~2.
		Consequently, after duplicate removal in objective space, the first nondominated front among non-gap points contains at most one representative for each suffix class $b\in\{0,\ldots,k\}$, and therefore has size at most $k+1$.
		
		Moreover, any non-gap representative with maximum prefix value among $S$ belongs to the first nondominated front.
		Hence, if $N\ge k+1$, Party~2 environmental selection preserves all first-front non-gap representatives, including any current maximum-prefix representative.
	\end{lemma}
	
	\begin{proof}
		For every non-gap point $\mathbf{x}\notin\mathcal H$, we have
		\[
		f_{2,1}(\mathbf{x}) = i(\mathbf{x}) + k - b(\mathbf{x}),
		\qquad
		f_{2,2}(\mathbf{x}) = i(\mathbf{x}) + b(\mathbf{x}).
		\]
		For a fixed suffix class $b$, both objectives increase strictly with $i$.
		Therefore, after duplicate removal in objective space, at most the largest-$i$ representative of each suffix class can appear in the first nondominated front.
		Since $b\in\{0,\ldots,k\}$, the non-gap part of the first front has size at most $k+1$.
		
		It remains to show the maximum-prefix claim.
		Let $\mathbf{x}$ be a non-gap representative with maximum prefix value $i(\mathbf{x})=J$.
		Suppose that another non-gap point $\mathbf{y}$ dominates $\mathbf{x}$ under Party~2.
		Then
		\[
		i(\mathbf{y})+k-b(\mathbf{y}) \ge J+k-b(\mathbf{x}),
		\qquad
		i(\mathbf{y})+b(\mathbf{y}) \ge J+b(\mathbf{x}).
		\]
		Since $i(\mathbf{y})\le J$, these inequalities imply respectively
		\[
		b(\mathbf{y})\le b(\mathbf{x})
		\quad\text{and}\quad
		b(\mathbf{y})\ge b(\mathbf{x}).
		\]
		Hence $b(\mathbf{y})=b(\mathbf{x})$, and then both inequalities force $i(\mathbf{y})=J$.
		Thus $\mathbf{y}$ cannot strictly dominate $\mathbf{x}$.
		Therefore $\mathbf{x}$ is in the first nondominated front.
		Since this front has size at most $k+1$, it is fully retained when $N\ge k+1$.
	\end{proof}
	
	\begin{lemma}[Reaching $\mathcal D_{\mathrm{pre}}$]
		\label{lem:prefix-donor}
		We have
		\[
		\mathbb{E}[T_{\mathrm{pre}}]=O(n\log n).
		\]
	\end{lemma}
	
	\begin{proof}
		Let $\tau_{\mathrm{ng}}$ denote the first generation in which the Party~2 immigrant is a non-gap point.
		Since a uniformly random bit string belongs to the gap set with probability $|\mathcal H|/2^n=k/2^n$, we have
		\[
		\Pr(\mathbf{r}_t^2\notin\mathcal H)=1-\frac{k}{2^n}=1-o(1),
		\]
		and therefore
		\[
		\mathbb{E}[\tau_{\mathrm{ng}}]=O(1).
		\]
		In generation $\tau_{\mathrm{ng}}$, the immigrant $\mathbf{r}_{\tau_{\mathrm{ng}}}^2$ is inserted into the augmented pool $\widehat P_{\tau_{\mathrm{ng}}}^2$.
		Since every non-gap point has at least one strictly positive Party~2 objective value whereas every gap point has objective vector $(0,0)$, each non-gap point dominates every gap point.
		Hence at least one non-gap point survives the subsequent environmental selection.
		From generation $\tau_{\mathrm{ng}}+1$ onward, Party~2 therefore always contains at least one non-gap point.
		
		For $t\ge \tau_{\mathrm{ng}}+1$ with $t<T_{\mathrm{pre}}$, define
		\[
		J_t
		:=
		\max\{i(\mathbf{x}) : \mathbf{x}\in P_t^2\setminus\mathcal H\},
		\]
		and let $\mathbf{x}_t\in P_t^2\setminus\mathcal H$ satisfy $i(\mathbf{x}_t)=J_t$.
		By Lemma~\ref{lem:party2-representatives}, $\mathbf{x}_t$ is a first-front non-gap representative and the non-gap part of the first front has size at most $k+1$.
		Since $N\ge k+1$, $\mathbf{x}_t$ survives environmental selection unless another non-gap representative with a larger prefix value is generated.
		Therefore $(J_t)_{t\ge \tau_{\mathrm{ng}}+1}$ is nondecreasing before $T_{\mathrm{pre}}$.
		
		Now fix a generation with $J_t=j<n-2k+1$.
		Then $\mathbf{x}_t$ has exactly $n-k-j$ zero-bits in its prefix.
		Consider one fixed offspring-generation step in Party~2.
		Since the augmented parent pool has size at most $N+1$, the probability that the primary-parent tournament samples $\mathbf{x}_t$ twice, and therefore selects $\mathbf{x}_t$, is at least $(N+1)^{-2}$.
		Conditioned on this, crossover is skipped with probability $1-p_c$, so the intermediate offspring equals $\mathbf{x}_t$.
		From there, mutation flips exactly one of the $n-k-j$ missing prefix zero-bits and no other bit with probability
		\[
		\frac{n-k-j}{n}\left(1-\frac{1}{n}\right)^{n-1}
		=
		\Omega\!\left(\frac{n-k-j}{n}\right).
		\]
		Since $j+1\le n-2k+1<n-k$, the offspring is still outside the gap set.
		Thus, in one fixed offspring-generation step of Party~2, the probability of creating a non-gap offspring $\mathbf{y}$ with
		\[
		i(\mathbf{y})=j+1,
		\qquad
		b(\mathbf{y})=b(\mathbf{x}_t)
		\]
		is at least
		\[
		c \cdot \frac{n-k-j}{n}
		\]
		for a constant $c>0$ depending only on $N$ and $p_c$.
		
		The offspring $\mathbf{y}$ belongs to the same suffix class as $\mathbf{x}_t$ and has a strictly larger prefix value.
		By Lemma~\ref{lem:party2-representatives}, either $\mathbf{y}$ becomes a first-front representative and is preserved, or an even larger-prefix non-gap representative is preserved.
		Therefore
		\[
		\Pr(J_{t+1}\ge j+1 \mid J_t=j)
		\ge
		c \cdot \frac{n-k-j}{n}.
		\]
		The expected waiting time to increase the current best prefix value from $j$ to at least $j+1$ is therefore
		\[
		O\!\left(\frac{n}{n-k-j}\right).
		\]
		Summing over all levels until $n-2k+1$ yields
		\[
		\mathbb{E}[T_{\mathrm{pre}}-\tau_{\mathrm{ng}}]
		=
		O\!\left(
		\sum_{j=0}^{n-2k}
		\frac{n}{n-k-j}
		\right)
		=
		O\!\left(
		n\sum_{m=k}^{n-k}\frac{1}{m}
		\right)
		=
		O(n\log n).
		\]
		Together with $\mathbb{E}[\tau_{\mathrm{ng}}]=O(1)$, this proves the claim.
	\end{proof}
	
	\begin{lemma}[Completion after boundary crossover]
		\label{lem:template-completion}
		Suppose that $T_{\mathrm{pre}}<\infty$.
		Then, starting from generation $T_{\mathrm{pre}}$, the expected additional number of generations until at least one element of
		\[
		\mathcal{PS}_{\mathrm{com}}=\{1^{n-k}0^k,\,1^n\}
		\]
		has been evaluated is $O(n)$.
	\end{lemma}
	
	\begin{proof}
		Fix a generation $t\ge T_{\mathrm{pre}}$ before any common solution has been evaluated.
		Let
		\[
		J_t
		:=
		\max\{i(\mathbf{x}) : \mathbf{x}\in P_t^2\setminus\mathcal H\}.
		\]
		By the definition of $T_{\mathrm{pre}}$, we have $J_t\ge n-2k+1$, and hence at most $k-1=O(1)$ prefix zero-bits remain to be corrected.
		
		Consider a generation in which Party~2 contains a maximum-prefix non-gap representative $\mathbf{x}_t$ with $i(\mathbf{x}_t)=J_t$.
		By Lemma~\ref{lem:party2-representatives}, such a representative is in the first nondominated front and is preserved unless a representative with a strictly larger prefix value is generated.
		
		Consider the following successful boundary-crossover event in one offspring-generation step of Party~2.
		\begin{enumerate}
			\item The primary-parent tournament selects $\mathbf{x}_t$ from $\widehat P_t^2$.
			\item Inter-party mating is selected.
			\item The secondary-parent tournament selects the current Party~1 immigrant from $\widehat P_t^1$.
			\item The selected immigrant belongs to $\mathcal T_0\cup\mathcal T_1$.
			\item One-point crossover uses the cut position $n-k$.
			\item Mutation flips no bit.
		\end{enumerate}
		The joint probability of these events is
		\[
		\Omega(1)\cdot p_g\cdot \Omega(1)\cdot 2^{1-2k}\cdot p_c\cdot \frac{1}{n-1}\cdot
		\left(1-\frac{1}{n}\right)^n
		=
		\Omega\!\left(\frac{1}{n}\right),
		\]
		where the two $\Omega(1)$ factors come from binary tournaments over parent pools of size at most $N+1$.
		
		Conditioned on this event, the offspring inherits the Party~2 prefix of $\mathbf{x}_t$ and a suffix $0^k$ or $1^k$ from the Party~1 immigrant.
		If the inherited suffix is $0^k$, then the offspring $\mathbf{y}$ satisfies
		\[
		b(\mathbf{y})=k,
		\qquad
		i(\mathbf{y})=J_t.
		\]
		If the inherited suffix is $1^k$, then
		\[
		b(\mathbf{y})=0,
		\qquad
		i(\mathbf{y})=J_t.
		\]
		In either case, $\mathbf{y}$ is outside the gap set.
		
		If $J_t=n-k$, then $\mathbf{y}$ is either $1^{n-k}0^k$ or $1^n$, and a common Pareto-optimal solution has been evaluated.
		Otherwise, $\mathbf{y}$ has a safe suffix $b(\mathbf{y})\in\{0,k\}$ and at most $k-1$ missing prefix one-bits.
		
		We next argue that the remaining completion takes expected $O(n)$ generations.
		Suppose that a representative $\mathbf{z}$ with $b(\mathbf{z})\in\{0,k\}$ and $i(\mathbf{z})=j<n-k$ is currently available.
		In one fixed offspring-generation step, selecting $\mathbf{z}$ as the primary parent, skipping crossover, and flipping exactly one missing prefix zero-bit and no other bit occurs with probability
		\[
		\Omega\!\left(\frac{n-k-j}{n}\right).
		\]
		The resulting offspring has the same suffix and prefix value $j+1$.
		If this offspring is retained, the safe-suffix prefix value increases by one.
		If the current representative with $b\in\{0,k\}$ is not retained before such an improvement occurs, then it must have been dominated by a non-gap representative with a strictly larger prefix value. Indeed, for \(b=0\) or \(b=k\), a non-gap point with prefix value at most \(j\) cannot dominate a point with the same prefix value \(j\) and the safe suffix \(b\).
		Thus, losing such a representative can only be charged to an increase in the maximum prefix value.
		
		Since after $T_{\mathrm{pre}}$ there are at most $k-1=O(1)$ missing prefix one-bits, at most $O(1)$ such prefix-level increases are needed.
		At each remaining prefix level, either another successful boundary crossover is sampled or a one-bit prefix improvement is obtained within expected $O(n)$ generations.
		Therefore a common Pareto-optimal solution is evaluated within expected $O(n)$ additional generations.
	\end{proof}
	
	\begin{theorem}[CPR on MP-JCG]
		\label{thm:cpr-mpjcg}
		Assume $k=O(1)$, $N\ge k+1$, $p_g=\Omega(1)$, and that $p_c$ is bounded away from both $0$ and $1$.
		For the analytical CPR-NSGA-II variant defined in this subsection, the expected number of generations until both common Pareto-optimal solutions of MP-JCG have been discovered satisfies
		\[
		\mathbb{E}[T_{\mathrm{com}}]=O(n\log n).
		\]
	\end{theorem}
	
	\begin{proof}
		By Lemma~\ref{lem:prefix-donor}, Party~2 reaches $\mathcal D_{\mathrm{pre}}$ in expected $O(n\log n)$ generations.
		Then Lemma~\ref{lem:template-completion} yields, from that point onward, expected $O(n)$ additional generations until one of the two common solutions is evaluated.
		
		Suppose first that the discovered common solution is $1^{n-k}0^k$.
		This point satisfies $i=n-k$ and $b=k$.
		It is a Party~2 Pareto-optimal solution and therefore remains represented in the Party~2 first front under the party-wise population update.
		Using this full-prefix representative as the Party~2 primary parent and a Party~1 immigrant from $\mathcal T_1$, the same boundary-crossover argument as in Lemma~\ref{lem:template-completion} creates $1^n$ in expected $O(n)$ additional generations.
		
		The case in which the first discovered common solution is $1^n$ is symmetric.
		Then $1^n$ remains represented in the Party~2 first front, and a Party~1 immigrant from $\mathcal T_0$ yields $1^{n-k}0^k$ in expected $O(n)$ additional generations.
		
		Finally, once a common solution has been evaluated, it is inserted into the archive.
		The archive update is monotone on common Pareto-optimal solutions because if $\mathbf{x}\in\mathcal{PS}_{\mathrm{com}}$, then no search point can multi-party dominate $\mathbf{x}$; otherwise, it would dominate $\mathbf{x}$ for at least one party, contradicting $\mathbf{x}\in\mathcal{PS}_m$ for every party $m$.
		Hence both solutions remain in $A_{\mathrm{com}}$ once discovered.
		
		Combining the three parts gives
		\[
		\mathbb{E}[T_{\mathrm{com}}]
		=
		O(n\log n)+O(n)+O(n)
		=
		O(n\log n).
		\]
	\end{proof}
	
	\begin{corollary}[FE complexity]
		\label{cor:cpr-mpjcg-fe}
		Under the assumptions of Theorem~\ref{thm:cpr-mpjcg}, the expected fitness-evaluation complexity until
		\[
		\{1^{n-k}0^k,1^n\}\subseteq A_{\mathrm{com}}
		\]
		is also
		\[
		O(n\log n).
		\]
	\end{corollary}
	
	\begin{proof}
		Apart from initialization, each generation newly evaluates $2N+2$ candidates: two immigrants and $N$ offspring for each of the two parties.
		Since $N=O(1)$, the number of candidate evaluations per generation is $\Theta(1)$.
		The initialization cost is $2N=O(1)$ and is dominated by the runtime bound.
		Therefore the fitness-evaluation complexity is asymptotically equivalent to the generation complexity established in Theorem~\ref{thm:cpr-mpjcg}.
	\end{proof}
	
	\subsection{Runtime Analysis of a Payoff-Guided Baseline}
	\label{subsec:payoff-mpjcg}
	
	Sun et al.~\cite{sun2025runtime} analyzed a single-population mutation-based framework for multi-party multi-objective optimization.
	For their artificial benchmark BPAOAZ, they introduced a random-party process and a payoff-guided process.
	The latter serves as a stronger analytical comparator under more informed party selection.
	We adopt the same analytical role on MP-JCG.
	
	Since the original payoff in~\cite{sun2025runtime} is defined for the state-space structure of BPAOAZ, we instantiate a problem-specific payoff tailored to MP-JCG.
	The resulting process is used as a mutation-based theoretical baseline for comparison with CPR-NSGA-II in Section~\ref{subsec:cpr-mpjcg}.
	Its purpose is to quantify what can still be achieved by payoff-guided party interaction without cross-party recombination.
	
	Among the two common solutions of MP-JCG,
	\[
	\mathcal{PS}_{\mathrm{com}}=\{1^{n-k}0^k,\,1^n\},
	\]
	the point $1^{n-k}0^k$ can be reached without crossing the gap.
	By contrast, the discovery of $1^n$ requires bypassing the rejected gap layer induced by
	\[
	\mathcal H
	=
	\left\{
	\mathbf{x}\in\{0,1\}^n :
	|\mathbf{x}_1^{\,n-k}|_1=n-k,\ 
	|\mathbf{x}_{n-k+1}^{\,n}|_0=1
	\right\}.
	\]
	Accordingly, the runtime bottleneck of the payoff-guided baseline is captured by the first hitting time of $1^n$.
	
	The payoff-guided process serves as a strong mutation-only comparator---it uses problem-specific directional information to reward prefix completion and to penalize the rejected suffix layer $b=1$. Thus, any remaining bottleneck can be attributed to the absence of recombination rather than to the absence of local guidance.
	
	\begin{definition}[Structural potential and payoff on MP-JCG]
		\label{def:payoff-mpjcg}
		For each $\mathbf{x}\in\{0,1\}^n$, define
		\[
		u(\mathbf{x}) := (n-k)-|\mathbf{x}_1^{\,n-k}|_1,
		\qquad
		b(\mathbf{x}) := |\mathbf{x}_{n-k+1}^{\,n}|_0.
		\]
		Thus, $u(\mathbf{x})$ is the number of missing 1-bits in the prefix, and $b(\mathbf{x})$ is the number of 0-bits in the suffix.
		
		Let
		\[
		\gamma(b)
		:=
		\begin{cases}
			0, & b=0,\\
			3, & b=1,\\
			2, & b=2,\\
			b, & b\ge 3.
		\end{cases}
		\]
		The structural potential of $\mathbf{x}$ is defined as
		\[
		\Psi(\mathbf{x})
		:=
		u(\mathbf{x})+\gamma(b(\mathbf{x})).
		\]
		For two search points $\mathbf{x},\mathbf{x}'\in\{0,1\}^n$, the payoff of replacing $\mathbf{x}$ by $\mathbf{x}'$ is
		\[
		\pi_{\mathbf{x},\mathbf{x}'}
		:=
		\Psi(\mathbf{x})-\Psi(\mathbf{x}').
		\]
		A mutation step is accepted if and only if $\pi_{\mathbf{x},\mathbf{x}'}>0$, that is, if it strictly decreases the structural potential.
	\end{definition}
	
	Definition~\ref{def:payoff-mpjcg} is a problem-specific instantiation of a positive-payoff rule.
	The prefix term $u(\mathbf{x})$ measures progress toward the full-prefix structure required by Party~2, whereas the suffix term $\gamma(b(\mathbf{x}))$ is chosen so that moving from suffix-zero count $b=2$ to the gap layer $b=1$ is penalized, while the direct jump from $b=2$ to $b=0$ is rewarded.
	
	\begin{algorithm}[H]
		\caption{Payoff-guided baseline on MP-JCG}
		\label{alg:GEMPMO_payoff}
		\begin{algorithmic}[1]
			\STATE Choose an individual $\mathbf{x}$ uniformly at random from $\{0,1\}^n$
			\STATE $P \leftarrow \{\mathbf{x}\}$
			\WHILE{stopping criterion is not met}
			\STATE Select $\mathbf{x}\in P$
			\STATE Generate $\mathbf{x}'$ from $\mathbf{x}$ by standard bit mutation with rate $1/n$
			\IF{$\pi_{\mathbf{x},\mathbf{x}'}>0$}
			\STATE $P \leftarrow \{\mathbf{x}'\}$
			\ENDIF
			\ENDWHILE
		\end{algorithmic}
	\end{algorithm}
	
	Let $\mathbf{x}_t$ denote the search point contained in the population after iteration $t$.
	We define the first hitting time of $1^n$ by
	\[
	T_{\mathrm{pay}}
	:=
	\inf\{t\ge 0 : \mathbf{x}_t = 1^n\}.
	\]
	
	\begin{lemma}[Bottleneck entry]
		\label{lem:payoff-bottleneck}
		Let
		\[
		\tau_2
		:=
		\inf\{t\ge 0 : u(\mathbf{x}_t)=0 \text{ and } b(\mathbf{x}_t)=2\}.
		\]
		Then there exists a constant $\delta>0$, independent of $n$, such that
		\[
		\Pr(\tau_2<T_{\mathrm{pay}})\ge \delta
		\]
		for all sufficiently large $n$.
	\end{lemma}
	
	\begin{proof}
		We condition on the event that the initial suffix-zero count equals two. Since $k\ge2$ is constant,
		\[
		\Pr(b(\mathbf{x}_0)=2)=\binom{k}{2}2^{-k}=\Omega(1).
		\]
		It remains to show that, conditioned on this event, the process completes the prefix while keeping $b=2$ with probability bounded below by a positive constant.
		
		Consider a state with $b=2$ and $u=s>0$. A one-bit mutation that flips one of the $s$ missing prefix bits and no other bit is accepted and occurs with probability
		\[
		p_s=\Omega\!\left(\frac{s}{n}\right).
		\]
		We call such a mutation a good prefix correction. We next bound accepted suffix-changing mutations before the next good prefix correction. A direct move from $b=2$ to $b=0$ requires flipping the two remaining suffix zero-bits and has probability $O(n^{-2})$ in one iteration. A move from $b=2$ to $b=1$ can be accepted only if the mutation also decreases $u$ by at least two, because $\gamma(1)-\gamma(2)=1$. The probability of such a mutation is $O(s^2/n^3)$. Other accepted suffix-changing mutations are bounded by the same order because $k=O(1)$. Hence the probability of an accepted suffix-changing event in one iteration is
		\[
		q_s=O\!\left(\frac{1}{n^2}+\frac{s^2}{n^3}\right).
		\]
		A standard competing-events bound gives
		\[
		\Pr(\text{a suffix-changing accepted event occurs before the next good prefix correction})
		\le
		O\!\left(\frac{q_s}{p_s}\right)
		=
		O\!\left(\frac{1}{sn}+\frac{s}{n^2}\right).
		\]
		Therefore, the probability that the process performs a good prefix correction before any accepted suffix-changing event at every level $s=1,\ldots,u(\mathbf{x}_0)$ is bounded below by
		\[
		\prod_{s=1}^{u(\mathbf{x}_0)}
		\left(
		1-
		O\!\left(\frac{1}{sn}+\frac{s}{n^2}\right)
		\right)
		=
		\Omega(1),
		\]
		for all sufficiently large $n$. On this event, the process reaches a state with $u=0$ and $b=2$ before hitting $1^n$. Combining this event with $\Pr(b(\mathbf{x}_0)=2)=\Omega(1)$ proves the claim.
	\end{proof}
	
	\begin{lemma}[Final jump]
		\label{lem:payoff-final-jump}
		Conditioned on $\tau_2<T_{\mathrm{pay}}$, the expected additional number of iterations until $1^n$ is first reached satisfies
		\[
		\mathbb{E}[T_{\mathrm{pay}}-\tau_2 \mid \tau_2<T_{\mathrm{pay}}]=\Theta(n^2).
		\]
	\end{lemma}
	
	\begin{proof}
		At time $\tau_2$, the process is at a point with full prefix and exactly two 0-bits in the suffix.
		A one-bit mutation flipping exactly one of these two suffix 0-bits would move the process to the gap layer with suffix-zero count $b=1$.
		By Definition~\ref{def:payoff-mpjcg},
		\[
		\gamma(2)=2
		\qquad\text{and}\qquad
		\gamma(1)=3,
		\]
		so such a transition increases the structural potential and is therefore rejected.
		
		Therefore, before $1^n$ can be reached from this bottleneck state, some mutation must flip both remaining suffix 0-bits in the same iteration. The probability that a mutation flips these two specified bits, regardless of what happens to the other bits, is at most $1/n^2$. Hence the expected waiting time before any accepted exit that can remove the bottleneck is $\Omega(n^2)$.
		
		For the upper bound, the mutation that flips exactly these two suffix 0-bits and leaves all other bits unchanged occurs with probability
		\[
		\frac{1}{n^2}\left(1-\frac{1}{n}\right)^{n-2}
		=
		\Theta\!\left(\frac{1}{n^2}\right).
		\]
		This mutation directly generates $1^n$. Thus the expected additional hitting time is also $O(n^2)$.
		Combining the lower and upper bounds proves the claim.
	\end{proof}
	
	\begin{theorem}[Payoff baseline]
		\label{thm:payoff-mpjcg}
		Let $k=O(1)$.
		Under Algorithm~\ref{alg:GEMPMO_payoff}, the expected time to reach $1^n$ on MP-JCG satisfies
		\[
		\mathbb{E}[T_{\mathrm{pay}}]=\Theta(n^2).
		\]
	\end{theorem}
	
	\begin{proof}
		We first derive the upper bound.
		Consider the structural potential $\Psi$ from Definition~\ref{def:payoff-mpjcg}.
		Before the process reaches a state with
		\[
		u(\mathbf{x})=0
		\qquad\text{and}\qquad
		b(\mathbf{x})\in\{0,1,2\},
		\]
		either $u(\mathbf{x})>0$ or $b(\mathbf{x})\ge 3$ holds.
		If $u(\mathbf{x})>0$, then flipping exactly one missing prefix bit and no other bit decreases $u$ by one, leaves $b$ unchanged, decreases $\Psi$ by one, and is accepted. This event has probability at least
		\[
		\frac{1}{n}\left(1-\frac1n\right)^{n-1}
		=
		\Omega\!\left(\frac1n\right).
		\]
		If $u(\mathbf{x})=0$ and $b(\mathbf{x})\ge3$, then flipping exactly one suffix zero-bit to one and no other bit decreases $\gamma(b)$ by one and is accepted, again with probability $\Omega(1/n)$ because \(k=O(1)\).
		
		Thus, whenever the process has not yet reached a state with \(u=0\) and \(b\in\{0,1,2\}\), the expected waiting time for an accepted mutation that decreases \(\Psi\) by at least one is \(O(n)\).
		Since \(\Psi\) is integer-valued and at most \(n+O(1)\) for \(k=O(1)\), the expected time to reach such a state is \(O(n^2)\).
		
		If $b(\mathbf{x})=0$, then $\mathbf{x}=1^n$ and the target has already been reached.
		If $b(\mathbf{x})=1$, then one remaining suffix 0-bit must be flipped to $1$, which occurs with probability
		\[
		\frac{1}{n}\left(1-\frac{1}{n}\right)^{n-1}
		=
		\Theta\!\left(\frac{1}{n}\right),
		\]
		and is accepted because $\gamma(1)=3$ and $\gamma(0)=0$.
		Hence this case requires expected $O(n)$ additional iterations.
		If $b(\mathbf{x})=2$, then Lemma~\ref{lem:payoff-final-jump} gives expected $\Theta(n^2)$ additional iterations.
		Therefore
		\[
		\mathbb{E}[T_{\mathrm{pay}}]=O(n^2).
		\]
		
		For the lower bound, Lemma~\ref{lem:payoff-bottleneck} shows that the process enters the bottleneck state $(u,b)=(0,2)$ before hitting $1^n$ with probability at least a positive constant.
		Conditioned on this event, Lemma~\ref{lem:payoff-final-jump} yields an additional expected waiting time of $\Theta(n^2)$.
		Hence
		\[
		\mathbb{E}[T_{\mathrm{pay}}]
		\ge
		\Pr(\tau_2<T_{\mathrm{pay}})
		\cdot
		\mathbb{E}[T_{\mathrm{pay}}-\tau_2 \mid \tau_2<T_{\mathrm{pay}}]
		=
		\Omega(n^2).
		\]
		Combining the upper and lower bounds proves the claim.
	\end{proof}
	
	\begin{corollary}[FE complexity]
		\label{cor:payoff-mpjcg-fe}
		Under the assumptions of Theorem~\ref{thm:payoff-mpjcg}, the expected fitness-evaluation complexity until $1^n$ is first reached is also
		\[
		\Theta(n^2).
		\]
	\end{corollary}
	
	\begin{proof}
		Algorithm~\ref{alg:GEMPMO_payoff} evaluates exactly one offspring in each iteration.
		Therefore the number of iterations and the number of fitness evaluations are asymptotically identical.
	\end{proof}
	
	Theorem~\ref{thm:payoff-mpjcg} shows that payoff-guided party interaction alone does not remove the mutation bottleneck induced by the gap structure of MP-JCG.
	The bottleneck remains the final transition to $1^n$, which requires a rare two-bit mutation.
	By contrast, Theorem~\ref{thm:cpr-mpjcg} shows that this bottleneck can be bypassed once cross-party recombination combines an almost complete prefix from Party~2 with a suffix block made available through Party~1.
	
	\subsection{Runtime Analysis of NSGA-II on Flattened Jump-Count with Gap}
	\label{subsec:fjcg}
	
	We next analyze NSGA-II on the flattened version of MP-JCG.
	In this subsection, we study a mutation-only NSGA-II baseline using the population-update convention introduced in Sec.~\ref{subsec:nsgaii}. In each offspring trial, one parent is selected uniformly from the current population and standard bit mutation with rate $1/n$ is applied.
	Throughout this subsection, we assume $2 \le k \le \lfloor n/2 \rfloor$ and $k=O(1)$.
	
	\begin{definition}[Flattened Jump-Count with Gap (F-JCG)]
		\label{def:F-JCG}
		Let $n \in \mathbb{N}$ and $k \in \{2,\ldots,\lfloor n/2 \rfloor\}$.
		The F-JCG is defined as the flattened version of MP-JCG obtained by applying the flattening operator in Definition~\ref{def:flatten}.
	\end{definition}
	
	For $\mathbf{x}\in\{0,1\}^n$, define the flattened objective vector
	\[
	\mathbf{f}(\mathbf{x})
	=
	(f_{1,1}(\mathbf{x}),f_{1,2}(\mathbf{x}),f_{2,1}(\mathbf{x}),f_{2,2}(\mathbf{x})).
	\]
	We write $\mathcal{PF}_{\mathrm{flat}}$ for the Pareto front of the corresponding $4$-objective problem.
	
	For $t\in\{k,\ldots,n-k\}$, define
	\[
	\mathbf{v}_t := (k+t,\;k+n-t,\;t,\;t+k).
	\]
	For $j\in\{2,\ldots,k-1\}$, define
	\[
	\mathbf{u}_j := (j,\;k+j,\;n-j,\;n-k+j).
	\]
	When \(k=2\), this family is empty. In addition, define
	\[
	\widehat{\mathbf{u}} := (1,\;k+1,\;n-1,\;n-k-1).
	\]
	
	\begin{proposition}[Flattened Pareto front]
		\label{prop:pf-fjcg}
		The Pareto front of F-JCG is
		\[
		\mathcal{PF}_{\mathrm{flat}}
		=
		\{\mathbf{v}_0,\mathbf{v}_n\}
		\cup
		\{\mathbf{v}_t : k\le t\le n-k\}
		\cup
		\{\mathbf{u}_j : 2\le j\le k-1\}
		\cup
		\{\widehat{\mathbf{u}}\},
		\]
		where
		\[
		\mathbf{v}_0=(k,n+k,0,k),
		\qquad
		\mathbf{v}_n=(n+k,k,n,n-k).
		\]
	\end{proposition}
	
	\begin{proof}
		We first identify the nondominated representatives.
		By Proposition~\ref{prop:ps1-mpjcg}, for every $t\in\{k,\ldots,n-k\}$, any search point with Hamming weight $t$ has Party~1 vector
		\[
		(k+t,\;k+n-t).
		\]
		If such a point additionally has all suffix bits equal to $0$, then it has Party~2 vector
		\[
		(t,\;t+k),
		\]
		and therefore realizes $\mathbf{v}_t$.
		
		The extreme points $0^n$ and $1^n$ realize $\mathbf{v}_0$ and $\mathbf{v}_n$, respectively.
		
		Now consider search points with Hamming weight $n-j$ for $2\le j\le k-1$.
		If all $j$ zero-bits lie in the suffix, the resulting flattened vector is
		\[
		(j,\;k+j,\;n-j,\;n-k+j)=\mathbf{u}_j.
		\]
		Further, the point
		\[
		1^{n-k-1}01^k
		\]
		realizes $\widehat{\mathbf{u}}=(1,k+1,n-1,n-k-1)$.
		
		It remains to justify that no other flattened objective vector is needed.
		For a fixed middle Hamming weight $t\in\{k,\ldots,n-k\}$, the first three components are fixed once $t$ is fixed, and the fourth component is maximized when all suffix bits are zero. Hence all noncanonical points of the same middle weight are dominated by a canonical representative realizing $\mathbf{v}_t$.
		For a high Hamming weight $n-j$ with $2\le j\le k-1$, the first two components and the third component are fixed, and the Party~2 second component is maximized when all $j$ zero-bits lie in the suffix. This gives $\mathbf{u}_j$ and dominates all other non-gap points of the same weight.
		For Hamming weight $n-1$, the full-prefix point with one suffix zero belongs to the gap layer and has Party~2 value $(0,0)$; it is dominated by the non-gap representative $1^{n-k-1}01^k$, which realizes $\widehat{\mathbf{u}}$.
		Points with Hamming weight between $1$ and $k-1$ are dominated by the middle representative realizing $\mathbf{v}_k$, while the two all-zero and all-one points remain as the extreme vectors $\mathbf{v}_0$ and $\mathbf{v}_n$.
		
		It remains to verify nondominance.
		Within the family $\{\mathbf{v}_t : k\le t\le n-k\}$, increasing $t$ improves the first, third, and fourth components while worsening the second, so these vectors are mutually non-dominated.
		The two extreme points $\mathbf{v}_0$ and $\mathbf{v}_n$ are also non-dominated, since each maximizes one Party~1 objective.
		
		For the family $\{\mathbf{u}_j : 2\le j\le k-1\}$, increasing $j$ improves the first, second, and fourth components while worsening the third component, so these vectors are mutually non-dominated.
		For any $t \in \{k,\ldots,n-k\}$ and any $j \in \{2,\ldots,k-1\}$, the vector $\mathbf{v}_t$ has strictly larger second component than $\mathbf{u}_j$, whereas $\mathbf{u}_j$ has strictly larger third component than $\mathbf{v}_t$; hence neither dominates the other.
		Moreover, $\widehat{\mathbf{u}}$ has third component $n-1$, which is larger than that of every middle vector $\mathbf{v}_t$ with $k\le t\le n-k$ and every $\mathbf{u}_j$ with $2\le j\le k-1$. It is not dominated by $\mathbf{v}_n$, because $\widehat{\mathbf{u}}$ has the larger second component $k+1>k$, and it is not dominated by $\mathbf{v}_0$, because $\widehat{\mathbf{u}}$ has the larger third component $n-1>0$. Thus $\widehat{\mathbf{u}}$ is also nondominated.
		Therefore the stated set is exactly $\mathcal{PF}_{\mathrm{flat}}$.
	\end{proof}
	
	\begin{corollary}[Front size]
		\label{cor:pfsize-fjcg}
		The Pareto-front size of F-JCG satisfies
		\[
		|\mathcal{PF}_{\mathrm{flat}}| = n-k+2.
		\]
	\end{corollary}
	
	\begin{proof}
		By Proposition~\ref{prop:pf-fjcg},
		\[
		|\mathcal{PF}_{\mathrm{flat}}|
		=
		2 + (n-2k+1) + (k-2) + 1
		=
		n-k+2.
		\]
	\end{proof}
	
	Since the flattened NSGA-II baseline below uses population size $N=c(n+1)$ with a constant $c>1$, we have \(N>|\mathcal{PF}_{\mathrm{flat}}|\) for all sufficiently large \(n\).
	The runtime argument below is a full-front coverage analysis under this front-preserving elitist convention. In particular, once a representative of a flattened Pareto-front vector has been generated, it is assumed to remain available for subsequent coverage steps.
	
	For the runtime analysis, we use the following canonical representatives:
	\[
	z_t := 1^t0^{n-t},
	\qquad t\in\{k,\ldots,n-k\},
	\]
	\[
	x^{(j)} := 1^{n-j}0^j,
	\qquad j\in\{2,\ldots,k-1\},
	\]
	and
	\[
	\widehat{x} := 1^{n-k-1}01^k.
	\]
	These representatives realize the front vectors $\mathbf{v}_t$, $\mathbf{u}_j$, and $\widehat{\mathbf{u}}$, respectively.
	
	\begin{lemma}[Middle initialization]
		\label{lem:init-middle-fjcg}
		Let $X_n \sim \mathrm{Bin}(n,\tfrac12)$.
		Then
		\[
		\Pr(k \le X_n \le n-k) \ge \frac{\binom{2k}{k}}{2^{2k}}.
		\]
	\end{lemma}
	
	\begin{proof}
		Let
		\[
		P_n := \Pr(k \le X_n \le n-k).
		\]
		For $n\ge 2k$, the sequence $(P_n)_{n\ge 2k}$ is nondecreasing.
		Indeed, if $X_{n+1}=X_n+B$ with $B\sim \mathrm{Bin}(1,\tfrac12)$ independent of $X_n$, then the event $\{k\le X_n\le n-k\}$ implies
		\[
		k\le X_{n+1}\le n+1-k.
		\]
		Hence $P_n \ge P_{2k}$.
		Since
		\[
		P_{2k}=\Pr(X_{2k}=k)=\frac{\binom{2k}{k}}{2^{2k}},
		\]
		the claim follows.
	\end{proof}
	
	\begin{lemma}[Middle representative]
		\label{lem:canonical-middle-fjcg}
		Suppose the current population contains a retained solution $\mathbf{x}$ with Hamming weight
		$t \in \{k,\ldots,n-k\}$.
		Then the algorithm generates the canonical representative \(z_t\) within \(O(n^2)\) expected generations, that is, within \(O(n^3)\) expected fitness evaluations.
	\end{lemma}
	
	\begin{proof}
		Fix a solution $\mathbf{x}$ with $|\mathbf{x}|_1=t \in \{k,\ldots,n-k\}$.
		If $\mathbf{x}=z_t$, there is nothing to prove.
		
		Otherwise, let $r\ge 1$ denote the number of one-bits in the suffix of length $k$.
		Since $|\mathbf{x}|_1=t\le n-k$, there are at least $r$ zero-bits in the prefix of length $n-k$.
		Choose one suffix one-bit and one prefix zero-bit.
		Flipping exactly these two bits preserves the Hamming weight and the first three objective components, while increasing \(f_{2,2}\) by \(2\).
		Hence such a two-bit swap yields a new point in the same weight layer with strictly larger fourth component, and is therefore accepted by the front-preserving elitist update.
		
		At most \(k\) such improving swaps are needed to transform \(\mathbf{x}\) into a canonical representative realizing \(\mathbf{v}_t\), because the suffix has length \(k\).
		In one generation, the retained representative \(\mathbf{x}\) is selected as a mutation parent with constant aggregate probability over the \(N=\Theta(n)\) offspring trials. Conditioned on this selection, a fixed improving two-bit swap is generated with probability \(\Theta(1/n^2)\).
		Therefore each improving swap takes \(O(n^2)\) expected generations, and the total time is
		\[
		O(kn^2)=O(n^2)
		\]
		expected generations because \(k=O(1)\).
		By the FE convention in Sec.~\ref{sec:preliminaries}, this is \(O(n^3)\) expected fitness evaluations.
	\end{proof}
	
	\begin{lemma}[Middle-front coverage]
		\label{lem:middle-front-fjcg}
		If the algorithm has generated the canonical representative \(z_t\) with \(t\in\{k,\ldots,n-k\}\), then it generates at least one representative for every middle front vector
		\[
		\{\mathbf{v}_i : k\le i\le n-k\}
		\]
		within \(O(n\log n)\) additional expected generations, that is, within \(O(n^2\log n)\) additional fitness evaluations.
	\end{lemma}
	
	\begin{proof}
		Starting from \(z_i\) with \(i<n-k\), flipping exactly one zero-bit in the prefix and leaving all other bits unchanged yields a search point whose suffix remains \(0^k\) and whose Hamming weight becomes \(i+1\).
		Hence the offspring realizes the front vector \(\mathbf{v}_{i+1}\).
		In one generation, selection of the retained representative \(z_i\) occurs with constant aggregate probability over the \(N=\Theta(n)\) offspring trials. Conditioned on this selection, the probability of the required one-bit mutation is at least
		\[
		\frac{n-k-i}{n}\left(1-\frac1n\right)^{n-1}
		=
		\Omega\!\left(\frac{n-k-i}{n}\right).
		\]
		Therefore the expected number of generations to generate some representative of \(\mathbf{v}_{i+1}\) from a representative of \(\mathbf{v}_i\) is
		\[
		O\!\left(\frac{n}{n-k-i}\right).
		\]
		Summing from \(i=t\) to \(i=n-k-1\) yields
		\[
		O\!\left(\sum_{i=t}^{n-k-1}\frac{n}{n-k-i}\right)
		=
		O(n\log n).
		\]
		
		Starting from a representative of \(\mathbf{v}_i\) with suffix \(0^k\), flipping exactly one one-bit in the prefix and leaving all other bits unchanged yields a representative of \(\mathbf{v}_{i-1}\).
		The expected number of generations to move from \(\mathbf{v}_i\) to \(\mathbf{v}_{i-1}\) is
		\[
		O\!\left(\frac{n}{i}\right).
		\]
		Summing from \(i=t\) down to \(i=k+1\) yields at most
		\[
		O\!\left(\sum_{i=k+1}^{n-k}\frac{n}{i}\right)
		=
		O(n\log n).
		\]
		
		Thus, after generating one middle canonical representative, the algorithm generates at least one representative for every vector \(\mathbf{v}_i\) with \(k\le i\le n-k\) within \(O(n\log n)\) additional expected generations.
		By the FE convention in Sec.~\ref{sec:preliminaries}, this corresponds to \(O(n^2\log n)\) additional fitness evaluations.
	\end{proof}
	
	\begin{theorem}[F-JCG coverage]
		\label{thm:runtime-fjcg}
		Let NSGA-II use population size
		\[
		N=c(n+1)
		\]
		with a constant \(c>1\), and let the variation step be instantiated by standard bit mutation with rate \(1/n\).
		Under the front-preserving coverage convention stated above, the expected number of fitness evaluations until the algorithm has generated at least one representative for every vector in \(\mathcal{PF}_{\mathrm{flat}}\) satisfies
		\[
		\mathbb{E}[T_{\mathrm{F\mbox{-}JCG}}]
		=
		O(n^3+n^{k+1}).
		\]
	\end{theorem}
	
	\begin{proof}
		By Lemma~\ref{lem:init-middle-fjcg}, a uniformly sampled bit string lies in the middle region
		\[
		\{\mathbf{x}\in\{0,1\}^n : k\le |\mathbf{x}|_1 \le n-k\}
		\]
		with probability at least a positive constant depending only on \(k\).
		Since the initial population has size \(N=\Theta(n)\), the probability that no such point appears in the initialization is exponentially small in \(n\).
		If this exceptional event occurs, then from any current search point a middle-region point can be generated by flipping at most \(k\) specified bits and leaving all other bits unchanged, because any Hamming weight outside \(\{k,\ldots,n-k\}\) differs from the nearest boundary weight by at most \(k\). This takes \(O(n^k)\) expected fitness evaluations, which is absorbed by the final bound.
		
		Thus, the algorithm obtains a middle-region representative within a time dominated by the asymptotic bounds below.
		By Lemma~\ref{lem:canonical-middle-fjcg}, it then generates a canonical representative \(z_t\) for some
		\(t\in\{k,\ldots,n-k\}\) within \(O(n^3)\) expected fitness evaluations.
		Then Lemma~\ref{lem:middle-front-fjcg} implies that all middle front vectors
		\[
		\{\mathbf{v}_i : k\le i\le n-k\}
		\]
		are generated within an additional \(O(n^2\log n)\) expected fitness evaluations.
		
		It remains to generate the remaining front vectors.
		Starting from the boundary representative
		\[
		z_{n-k}=1^{n-k}0^k,
		\]
		the vectors
		\[
		\mathbf{u}_{k-1},\,\mathbf{u}_{k-2},\,\ldots,\,\mathbf{u}_2
		\]
		are generated by successively flipping one suffix zero-bit to one.
		Since \(k=O(1)\), there are only \(O(1)\) such vectors.
		Each step takes \(O(n)\) expected generations and therefore \(O(n^2)\) expected fitness evaluations.
		Hence all vectors \(\mathbf{u}_j\) are generated within \(O(n^2)\) additional fitness evaluations.
		
		To generate the extreme vector \(\mathbf{v}_n\), it suffices to obtain \(1^n\) from \(z_{n-k}\).
		A conservative upper bound is obtained by requiring all remaining \(k\) suffix zero-bits to flip simultaneously and no other bit to change.
		Conditioned on selecting \(z_{n-k}\) as the mutation parent in one offspring trial, this mutation occurs with probability
		\[
		\left(\frac{1}{n}\right)^k
		\left(1-\frac{1}{n}\right)^{n-k}
		=
		\Theta(n^{-k}).
		\]
		Since the population size is \(N=\Theta(n)\), a single offspring trial selects this particular retained representative with probability \(\Theta(1/n)\). Hence the success probability in one fitness evaluation is \(\Theta(n^{-(k+1)})\), and the expected number of fitness evaluations to generate \(1^n\) is \(O(n^{k+1})\).
		Equivalently, the expected number of generations is \(O(n^k)\), and multiplying by \(N=\Theta(n)\) offspring evaluations per generation again gives \(O(n^{k+1})\) fitness evaluations.
		
		By symmetry, the same bound holds for generating \(0^n\) and hence the vector \(\mathbf{v}_0\) from a lower boundary representative of \(\mathbf{v}_k\).
		
		Finally, once \(1^n\) has been generated, flipping exactly one prefix bit from \(1\) to \(0\) yields the representative \(\widehat{x}\) of \(\widehat{\mathbf{u}}\).
		In a single offspring trial, selecting \(1^n\) as parent has probability \(\Theta(1/n)\), and conditioned on this selection, flipping one prefix bit and no other bit occurs with probability \(\Theta(1)\). Hence \(\widehat{x}\) is generated within \(O(n)\) expected fitness evaluations.
		
		Combining the above bounds yields
		\[
		O(n^3) + O(n^2\log n) + O(n^2) + O(n^{k+1}) + O(n)
		=
		O(n^3+n^{k+1}),
		\]
		since \(O(n^3)\) dominates \(O(n^2\log n)\).
		This proves the claim.
	\end{proof}
	
	\subsection{Implications of the MP-JCG Analysis}
	\label{subsec:mpjcg-discussion}
	
	The results of this section support two conclusions about evolutionary search in MPMOPs.
	
	First, treating MPMOPs as common-solution search problems is theoretically meaningful.
	On MP-JCG, the common Pareto set is
	\[
	\mathcal{PS}_{\mathrm{com}}=\{1^{n-k}0^k,\,1^n\},
	\]
	whereas the flattened formulation F-JCG has Pareto-front size \(n-k+2\).
	This difference is not merely representational.
	It changes the search target.
	In the multi-party formulation, the target consists of solutions that are Pareto-optimal for both parties simultaneously.
	In the flattened formulation, full-front coverage requires additional nondominated vectors that are not common Pareto-optimal solutions.
	Thus, under a full-front coverage criterion, flattening introduces a coverage burden that is absent from the original common-solution search problem.
	The comparison in this section therefore shows that runtime analysis for flattened Pareto-front approximation cannot be used directly as a proxy for runtime analysis of common-solution discovery.
	
	Second, cross-party recombination is theoretically meaningful because it changes how common solutions are constructed.
	The payoff-guided baseline already uses problem-specific directional information, but it still reaches \(1^n\) only through a rare mutation event and therefore requires \(\Theta(n^2)\) expected fitness evaluations.
	By contrast, CPR-NSGA-II discovers both common Pareto-optimal solutions in expected \(O(n\log n)\) fitness evaluations under the analyzed setting.
	The difference is not merely the use of two populations.
	It is the availability of complementary components under party-wise search.
	On MP-JCG, Party~2 supplies an almost complete prefix, while Party~1 can supply a suffix block of type \(0^k\) or \(1^k\).
	Cross-party recombination combines these components by a boundary one-point crossover, replacing the mutation bottleneck by a constructive recombination event.
	In this sense, cross-party recombination is not only a diversity-enhancing operator in the present setting, but also a mechanism for common-solution construction.
	
	Overall, the results on MP-JCG suggest that the essential object in MPMOP runtime analysis is neither the aggregated Pareto front nor crossover in isolation.
	Rather, it is the interaction between common-solution structure and cross-party complementarity.
	The former determines what should be found, and the latter determines whether recombination can reduce the cost of finding it.
	
	\section{Runtime Analysis on BPBOMST}
	\label{sec:bpbomst-analysis}
	
	The MP-JCG analysis in Section~\ref{sec:mpjcg-analysis} isolates the role of CPR in a pseudo-Boolean search space, where complementary components can be combined by one-point crossover. A graph-constrained combinatorial setting introduces a different difficulty. Feasible solutions are spanning trees, and recombination must preserve or restore global connectivity. The minimum spanning-tree setting therefore provides a structured case for studying how CPR combines cross-party edge information through recombination
	and repair.
	
	Multi-party spanning-tree problems have appeared in multi-agent network	design and social-choice-based graph optimization. Li et	al.~\cite{li2023multiagentMSTCover} considered a multiagent MST cover problem, where different agents evaluate the same edge set and the goal is	to retain a subgraph that supports each agent's optimal spanning tree.	Darmann et al.~\cite{darmann2011finding} studied socially best spanning	trees under aggregated edge preferences, while Darmann~\cite{darmann2016condorcet} investigated Condorcet spanning trees under majority comparisons among agents. These studies differ in their solution concepts, but they share the same structural feature that several	decision makers evaluate or compare spanning trees over a common graph.
	
	This section studies a bi-party bi-objective specialization of MPMOMST,	denoted BPBOMST. Each party evaluates the same spanning tree through its own	two edge-weight objectives. The target is not the full Pareto front of a flattened four-objective MST problem, but the common Pareto structure induced	by two party-wise bi-objective MST problems. The analysis asks whether CPR can construct a spanning tree that approximates the same common Pareto objective vector simultaneously for both parties and both objectives.
	
	In MP-JCG, recombination directly combines a prefix and a suffix. In BPBOMST, useful partial information is carried by edges and substructures of spanning trees. Cross-party recombination therefore first merges edge information from different party-wise search processes and then restores spanning-tree feasibility. The analysis below studies this mechanism through edge-union recombination followed by repair.
	
	\subsection{BPBOMST and Common Approximation}
	\label{subsec:bpbomst-definition}
	
	We first formalize the general MPMOMST setting and then specialize the common approximation target to BPBOMST. The problem is a minimization problem. We use \(p\) as the party index and reserve \(m=|E|\) for the number of graph edges.
	
	\begin{definition}[MPMOMST]
		\label{def:MPMOMST}
		Let \(G=(V,E)\) be a connected simple undirected graph with \(V=\{v_1,\dots,v_n\}\) and \(E=\{e_1,\dots,e_m\}\).
		Let there be \(M\) parties. For each party \(p\in[M]:=\{1,\dots,M\}\), let
		\(k_p\in\mathbb N\) denote the number of objectives associated with party
		\(p\). For the runtime analysis, all edge weights are positive integers. For
		each edge \(e\in E\) and party \(p\),
		\[
		\mathbf w^{(p)}(e)
		=
		\bigl(w^{(p)}_1(e),\dots,w^{(p)}_{k_p}(e)\bigr)
		\in
		\{1,\dots,w_{\max}\}^{k_p}.
		\]
		For any spanning tree \(T\subseteq E\), the objective vector of party \(p\)
		is
		\[
		F_p(T)
		=
		\sum_{e\in T}\mathbf w^{(p)}(e)
		=
		\left(
		\sum_{e\in T}w^{(p)}_1(e),\dots,
		\sum_{e\in T}w^{(p)}_{k_p}(e)
		\right).
		\]
		Let \(\mathcal T\) denote the set of all spanning trees of \(G\). For each
		party \(p\), let \(\mathcal{PS}_p\) denote the Pareto set induced by
		\(F_p\) on \(\mathcal T\). The common Pareto set is
		\[
		\mathcal{PS}_{\mathrm{com}}
		=
		\bigcap_{p=1}^{M}\mathcal{PS}_p .
		\]
		The MPMOMST problem is to identify spanning trees in
		\(\mathcal{PS}_{\mathrm{com}}\) when this set is nonempty.
	\end{definition}
	
	Since every spanning tree contains \(n-1\) edges and all edge weights are
	positive, each objective value satisfies
	\[
	n-1
	\le
	f_i^{(p)}(T)
	\le
	W,
	\qquad
	W:=(n-1)w_{\max}.
	\]
	The positivity assumption avoids degeneracy in ratio-based approximation
	statements, while \(W\) bounds the number of possible integer objective
	levels used in the runtime analysis.
	
	The runtime analysis focuses on the BPBOMST case with two parties and two
	objectives for each party. Thus, for \(p\in\{1,2\}\),
	\[
	F_p(T)
	=
	\bigl(f_1^{(p)}(T),f_2^{(p)}(T)\bigr).
	\]
	For a spanning tree \(T\), define its joint objective vector by
	\[
	Y(T)
	:=
	\bigl(
	f_1^{(1)}(T),f_2^{(1)}(T),
	f_1^{(2)}(T),f_2^{(2)}(T)
	\bigr).
	\]
	The common Pareto front is
	\[
	\mathcal{PF}_{\mathrm{com}}
	:=
	\{Y(T):T\in\mathcal{PS}_{\mathrm{com}}\}.
	\]
	This set contains objective vectors rather than spanning trees. Different
	common Pareto-optimal trees that induce the same four-dimensional objective vector
	are represented by the same point in \(\mathcal{PF}_{\mathrm{com}}\).
	
	\begin{definition}[Common approximation cover]
		\label{def:common-approx-cover}
		Let \(\alpha\ge 1\). A set of spanning trees
		\(\mathcal A\subseteq\mathcal T\) is an \(\alpha\)-common approximation cover
		of \(\mathcal{PF}_{\mathrm{com}}\) if, for every
		\(y\in\mathcal{PF}_{\mathrm{com}}\), there exists \(T\in\mathcal A\) such that
		\[
		Y(T)\le \alpha y
		\]
		component-wise.
	\end{definition}
	
	Definition~\ref{def:common-approx-cover} is the approximation target of this
	section. It differs from approximating each party-wise Pareto front
	separately, because two independently obtained approximation sets need not
	contain a tree that approximates the same common Pareto objective vector for both
	parties. It also differs from approximating the full Pareto front of the
	flattened four-objective MST problem, because the flattened front may contain
	nondominated objective vectors that are unrelated to
	\(\mathcal{PS}_{\mathrm{com}}\).
	
	\subsection{CPR-NSGA-II for BPBOMST}
	\label{subsec:cpr-nsgaii-bpbomst}
	
	This subsection gives the analysis-oriented CPR-NSGA-II mechanism used in the runtime theorem. The algorithm maintains two party-wise representative populations, \(\mathcal P_1^t\) and \(\mathcal P_2^t\), and one common archive, \(\mathcal A_{\mathrm{com}}^t\). The population \(\mathcal P_p^t\) preserves search information under party \(p\)'s bi-objective evaluation \(F_p\). The archive \(\mathcal A_{\mathrm{com}}^t\) records nondominated joint objective vectors under \(Y\) and supplies receiver trees for archive-assisted local search and CPR.
	
	Fix a deterministic canonical tie-breaking rule over spanning trees, such as
	the lexicographic order of edge-incidence vectors. For a candidate set \(Q\)
	and party \(p\), let \(\operatorname{Rep}_p(Q)\) contain the canonical
	spanning tree for each distinct vector in \(F_p(Q)\). Similarly, let
	\(\operatorname{Rep}_Y(Q)\) contain the canonical spanning tree for each
	distinct vector in \(Y(Q)\). The party-wise update is
	\[
	\mathcal P_p^{t+1}
	\leftarrow
	\textsc{PopulationUpdate}_p
	\left(
	\operatorname{Rep}_p(\mathcal P_p^t\cup Q),N_p
	\right),
	\]
	where nondominated sorting is computed under \(F_p\). The analyzed
	representative-pool version retains all nondominated party-wise
	objective-vector representatives. Since each party-wise objective vector has
	two positive integer components bounded by \(W\), the size of each such
	representative pool is at most \(W+1\). In the representative-pool theorem below, \(N_p\) is chosen large enough to retain this entire pool, and \(N_p\ge W+1\) suffices. A finite NSGA-II implementation with crowding distance can be used in practice, but the explicit runtime theorem is stated for this representative-pool variant.
	
	The common archive update \(\textsc{ArchiveUpdate}(\mathcal A,Q)\) keeps one canonical representative for each nondominated joint objective vector in \(Y(\mathcal A\cup Q)\). The archive records joint progress toward approximating \(\mathcal{PF}_{\mathrm{com}}\); it does not redefine the problem as a flattened four-objective MST.
	
	The analyzed algorithm uses party-wise local variation, archive-local
	variation, and archive-assisted CPR. Party-wise local variation selects a
	parent from \(\mathcal P_p^t\), applies a one-edge exchange, repairs the
	result into a spanning tree, and updates \(\mathcal P_p^t\) under \(F_p\).
	This mechanism generates the party-wise edge structures that later serve as
	CPR providers. Archive-local variation selects a receiver from
	\(\mathcal A_{\mathrm{com}}^t\), applies a one-edge exchange, repairs the
	result, and submits the offspring to the common archive. CPR selects a receiver \(T_r\in\mathcal A_{\mathrm{com}}^t\) and a provider \(T_s\in\mathcal P_1^t\cup\mathcal P_2^t\), forms the edge-union graph \(G[T_r\cup T_s]=(V,T_r\cup T_s)\), and applies \(\textsc{UniformRepair}(G[T_r\cup T_s])\), which samples a spanning tree uniformly at random from \(\mathcal T(G[T_r\cup T_s])\). Here \(\mathcal T(H)\) denotes the set of spanning trees of a graph \(H\). Since both \(T_r\) and \(T_s\) are spanning trees, \(G[T_r\cup T_s]\) contains at least one spanning tree. Uniform repair is used to obtain explicit success probabilities in the runtime proof.
	Practical implementations may use deterministic or heuristic repair, but the
	stated bound applies to the uniform repair instantiation.
	
	\begin{algorithm}[H]
		\caption{CPR-NSGA-II for BPBOMST}
		\label{alg:cnsga-ii-m3st}
		\begin{algorithmic}[1]
			\REQUIRE Population sizes \(N_1,N_2\), CPR probability \(p_g\)
			\ENSURE Common archive \(\mathcal A_{\mathrm{com}}\)
			
			\STATE Initialize \(\mathcal P_1^0\) and \(\mathcal P_2^0\) with feasible spanning trees.
			\STATE Evaluate all initial trees under \(F_1\) and \(F_2\).
			\STATE \(\mathcal P_p^0\leftarrow
			\textsc{PopulationUpdate}_p(\operatorname{Rep}_p(\mathcal P_p^0),N_p)\) for \(p\in\{1,2\}\).
			\STATE \(\mathcal A_{\mathrm{com}}^0
			\leftarrow
			\textsc{ArchiveUpdate}(\emptyset,\mathcal P_1^0\cup\mathcal P_2^0)\).
			
			\FOR{\(t=0,1,2,\ldots\)}
			\STATE \(Q_{\mathrm{pw}}^t\leftarrow\emptyset\).
			\FOR{\(p\in\{1,2\}\)}
			\STATE Select \(T_p\) uniformly from \(\mathcal P_p^t\).
			\STATE \(T_p'\leftarrow\textsc{Repair}(\textsc{OneEdgeExchange}(T_p))\).
			\STATE \(Q_{\mathrm{pw}}^t\leftarrow Q_{\mathrm{pw}}^t\cup\{T_p'\}\).
			\ENDFOR
			
			\STATE \(Q_{\mathrm{com}}^t\leftarrow\emptyset\).
			\STATE Draw \(r\sim U(0,1)\).
			\IF{\(r<p_g\)}
			\STATE Select \(T_r\) uniformly from \(\mathcal A_{\mathrm{com}}^t\).
			\STATE Select \(T_s\) uniformly from \(\mathcal P_1^t\cup\mathcal P_2^t\).
			\STATE \(T_c\leftarrow\textsc{UniformRepair}(G[T_r\cup T_s])\).
			\STATE \(Q_{\mathrm{com}}^t\leftarrow Q_{\mathrm{com}}^t\cup\{T_c\}\).
			\ELSE
			\STATE Select \(T_r\) uniformly from \(\mathcal A_{\mathrm{com}}^t\).
			\STATE \(T_c\leftarrow\textsc{Repair}(\textsc{OneEdgeExchange}(T_r))\).
			\STATE \(Q_{\mathrm{com}}^t\leftarrow Q_{\mathrm{com}}^t\cup\{T_c\}\).
			\ENDIF
			
			\STATE \(Q_t\leftarrow Q_{\mathrm{pw}}^t\cup Q_{\mathrm{com}}^t\).
			\STATE Evaluate all trees in \(Q_t\) under \(F_1\) and \(F_2\).
			\STATE \(\mathcal A_{\mathrm{com}}^{t+1}
			\leftarrow
			\textsc{ArchiveUpdate}(\mathcal A_{\mathrm{com}}^t,Q_t)\).
			\FOR{\(p\in\{1,2\}\)}
			\STATE \(\mathcal P_p^{t+1}
			\leftarrow
			\textsc{PopulationUpdate}_p
			\left(
			\operatorname{Rep}_p(\mathcal P_p^t\cup Q_t),N_p
			\right)\).
			\ENDFOR
			\ENDFOR
			
			\RETURN \(\mathcal A_{\mathrm{com}}\).
		\end{algorithmic}
	\end{algorithm}
	
	The party-wise local steps ensure that party-wise convex-front
	representatives can be generated by the algorithm itself. The archive-local
	step supplies local progress in the common-cover construction. The CPR step
	combines a common-archive receiver with a party-wise provider and uses
	uniform edge-union repair to obtain explicit probability bounds. The two
	layers are therefore kept separate. Party-wise populations preserve
	party-wise edge structures, while the common archive records joint
	objective-vector progress.
	
	\subsection{Approximation Analysis}
	\label{subsec:bpbomst-approximation}
	
	This subsection develops a layered support-cover argument for BPBOMST. The
	first layer maps the two normalized objectives of each party into one
	party-level scalar, yielding an auxiliary bi-objective problem. The support
	points of this auxiliary problem provide a factor-\(2\) bound in the
	party-level space. The second layer lifts this bound back to the original
	four objective components, incurring a projection-dependent loss. These
	auxiliary objectives are used only for analysis and are not computed by the
	algorithm.
	
	Let \(y=(y_{1,1},y_{1,2},y_{2,1},y_{2,2})\in\mathcal{PF}_{\mathrm{com}}\). For any spanning tree \(T\in\mathcal T\), define
	\[
	r_{p,i}(T\mid y)
	=
	\frac{f_i^{(p)}(T)}{y_{p,i}},
	\qquad
	p\in\{1,2\},\quad i\in\{1,2\}.
	\]
	Since all edge weights are positive, these ratios are well-defined.
	
	A party-level projection is a pair of positive, component-wise nondecreasing
	functions
	\[
	\phi=(\phi_1,\phi_2),
	\qquad
	\phi_p:\mathbb R_{>0}^2\to\mathbb R_{>0},
	\]
	satisfying
	\[
	\phi_p(1,1)=1,
	\qquad p\in\{1,2\}.
	\]
	For such a projection, define
	\[
	A_y^\phi(T)
	=
	\left(
	A_{1,y}^\phi(T),
	A_{2,y}^\phi(T)
	\right),
	\]
	where
	\[
	A_{p,y}^\phi(T)
	=
	\phi_p
	\left(
	r_{p,1}(T\mid y),
	r_{p,2}(T\mid y)
	\right).
	\]
	Let \(\mathcal{PF}_{A_y^\phi}\) be the Pareto front induced by
	\(A_y^\phi\) on \(\mathcal T\), and let
	\(\operatorname{supp}(\mathcal{PF}_{A_y^\phi})\) denote the support points on
	the lower-left convex hull of this front.
	
	\begin{lemma}[Two-dimensional support-cover bound]
		\label{lem:two-dimensional-support-cover}
		Let \(F\subset\mathbb R_{>0}^2\) be a finite Pareto front of a bi-objective
		minimization problem, and let \(\operatorname{supp}(F)\) be the set of
		support points on the lower-left convex hull of \(F\). For every \(z\in F\),
		there exists \(q\in\operatorname{supp}(F)\) such that
		\[
		q_1\le 2z_1,
		\qquad
		q_2\le 2z_2.
		\]
	\end{lemma}
	
	\begin{proof}
		If \(z\in\operatorname{supp}(F)\), the claim is immediate. Otherwise, \(z\)
		lies above a line segment between two adjacent support points
		\(q^a,q^b\in\operatorname{supp}(F)\). Hence there exists
		\(\theta\in[0,1]\) such that
		\[
		u=\theta q^a+(1-\theta)q^b
		\le z
		\]
		component-wise. If \(\theta\ge 1/2\), then
		\[
		\theta q_i^a\le u_i\le z_i
		\]
		for \(i=1,2\), and hence \(q_i^a\le 2z_i\) for both components. If
		\(\theta<1/2\), then \(1-\theta\ge 1/2\), and the same argument gives
		\(q_i^b\le 2z_i\) for both components.
	\end{proof}
	
	For a candidate tree \(T\) and a reference point
	\(y\in\mathcal{PF}_{\mathrm{com}}\), define the lifting loss of \(\phi\) by
	\[
	L_\phi(T\mid y)
	=
	\max_{p\in\{1,2\}}
	\max_{i\in\{1,2\}}
	\frac{r_{p,i}(T\mid y)}{A_{p,y}^\phi(T)}.
	\]
	This quantity measures the loss incurred when a party-level bound in the
	auxiliary two-dimensional space is lifted to the original four objectives.
	
	\begin{theorem}[Layered support-cover approximation]
		\label{thm:layered-support-cover}
		Let \(\phi=(\phi_1,\phi_2)\) be a party-level projection satisfying
		\(\phi_p(1,1)=1\) for \(p\in\{1,2\}\). For every
		\(y\in\mathcal{PF}_{\mathrm{com}}\), there exists a spanning tree \(R_y\)
		such that
		\[
		A_y^\phi(R_y)
		\in
		\operatorname{supp}(\mathcal{PF}_{A_y^\phi})
		\]
		and
		\[
		A_{p,y}^\phi(R_y)\le 2,
		\qquad p\in\{1,2\}.
		\]
		Moreover, every such support representative satisfies
		\[
		Y(R_y)
		\le
		2L_\phi(R_y\mid y)y
		\]
		component-wise.
	\end{theorem}
	
	\begin{proof}
		Fix \(y\in\mathcal{PF}_{\mathrm{com}}\). By definition, there exists
		\(T_y\in\mathcal{PS}_{\mathrm{com}}\) such that \(Y(T_y)=y\). Hence
		\[
		r_{p,i}(T_y\mid y)=1,
		\qquad
		p\in\{1,2\},\quad i\in\{1,2\}.
		\]
		Since \(\phi_p(1,1)=1\), we have
		\[
		A_y^\phi(T_y)=(1,1).
		\]
		If \((1,1)\) is Pareto-optimal under \(A_y^\phi\), then
		Lemma~\ref{lem:two-dimensional-support-cover} applied to \(z=(1,1)\) gives
		a support point \(q\) satisfying \(q_p\le 2\) for \(p\in\{1,2\}\). If
		\((1,1)\) is dominated under \(A_y^\phi\), then there exists
		\(z\in\mathcal{PF}_{A_y^\phi}\) with \(z\le(1,1)\). Applying
		Lemma~\ref{lem:two-dimensional-support-cover} to this \(z\) gives a support
		point \(q\) satisfying \(q_p\le 2z_p\le 2\). Since \(\mathcal T\) is finite,
		the support point \(q\) is realized by at least one spanning tree \(R_y\).
		
		For this tree,
		\[
		A_{p,y}^\phi(R_y)\le 2,
		\qquad p\in\{1,2\}.
		\]
		By the definition of \(L_\phi\),
		\[
		r_{p,i}(R_y\mid y)
		\le
		L_\phi(R_y\mid y)A_{p,y}^\phi(R_y)
		\le
		2L_\phi(R_y\mid y).
		\]
		Equivalently,
		\[
		f_i^{(p)}(R_y)
		\le
		2L_\phi(R_y\mid y)y_{p,i},
		\qquad
		p\in\{1,2\},\quad i\in\{1,2\}.
		\]
		This is exactly \(Y(R_y)\le 2L_\phi(R_y\mid y)y\).
	\end{proof}
	
	For each \(y\in\mathcal{PF}_{\mathrm{com}}\), define
	\[
	\mathcal S_\phi(y)
	=
	\left\{
	R\in\mathcal T:
	A_y^\phi(R)\in\operatorname{supp}(\mathcal{PF}_{A_y^\phi}),
	\ 
	A_{p,y}^{\phi}(R)\le2,\ p\in\{1,2\}
	\right\}.
	\]
	By Theorem~\ref{thm:layered-support-cover}, \(\mathcal S_\phi(y)\) is
	nonempty. Choose
	\[
	R_y^\star
	\in
	\arg\min_{R\in\mathcal S_\phi(y)}
	L_\phi(R\mid y),
	\]
	and define
	\[
	\Lambda_{\mathrm{eff}}^\phi
	=
	\max_{y\in\mathcal{PF}_{\mathrm{com}}}
	L_\phi(R_y^\star\mid y),
	\qquad
	\mathcal R_\phi^\star
	=
	\{R_y^\star:y\in\mathcal{PF}_{\mathrm{com}}\}.
	\]
	
	\begin{corollary}[Projection-induced common approximation cover]
		\label{cor:projection-induced-cover}
		The set \(\mathcal R_\phi^\star\) is a
		\(2\Lambda_{\mathrm{eff}}^\phi\)-common approximation cover of
		\(\mathcal{PF}_{\mathrm{com}}\).
	\end{corollary}
	
	\begin{proof}
		For any \(y\in\mathcal{PF}_{\mathrm{com}}\),
		Theorem~\ref{thm:layered-support-cover} gives
		\[
		Y(R_y^\star)
		\le
		2L_\phi(R_y^\star\mid y)y.
		\]
		By the definition of \(\Lambda_{\mathrm{eff}}^\phi\),
		\[
		L_\phi(R_y^\star\mid y)
		\le
		\Lambda_{\mathrm{eff}}^\phi.
		\]
		Hence
		\[
		Y(R_y^\star)
		\le
		2\Lambda_{\mathrm{eff}}^\phi y .
		\]
	\end{proof}
	
	We instantiate the projection by the symmetric average
	\[
	\phi_p^{\mathrm{avg}}(a,b)=\frac{a+b}{2}.
	\]
	Then
	\[
	A_{p,y}^{\mathrm{avg}}(T)
	=
	\frac12
	\left(
	r_{p,1}(T\mid y)+r_{p,2}(T\mid y)
	\right).
	\]
	For this projection, define
	\[
	\lambda(T\mid y)
	:=
	L_{\mathrm{avg}}(T\mid y).
	\]
	Equivalently,
	\[
	\lambda(T\mid y)
	=
	\max_{p\in\{1,2\}}
	\frac{
		\max\{r_{p,1}(T\mid y),r_{p,2}(T\mid y)\}
	}{
		\frac12
		\left(
		r_{p,1}(T\mid y)+r_{p,2}(T\mid y)
		\right)
	}.
	\]
	
	\begin{corollary}[Average projection]
		\label{cor:average-projection-cover}
		For the symmetric average projection,
		\[
		1\le \lambda(T\mid y)\le 2
		\]
		for all spanning trees \(T\) and all \(y\in\mathcal{PF}_{\mathrm{com}}\).
		Consequently, if \(\lambda_{\mathrm{eff}}:=\Lambda_{\mathrm{eff}}^{\mathrm{avg}}\), then \(1\le\lambda_{\mathrm{eff}}\le2\), and
		\(\mathcal R_{\mathrm{avg}}^\star\) is a
		\(2\lambda_{\mathrm{eff}}\)-common approximation cover of
		\(\mathcal{PF}_{\mathrm{com}}\).
	\end{corollary}
	
	\begin{proof}
		Let \(a,b>0\). Since \(\max\{a,b\}\) is at least the average and at most
		\(a+b\),
		\[
		1
		\le
		\frac{\max\{a,b\}}{(a+b)/2}
		\le
		2.
		\]
		Applying this inequality to \(a=r_{p,1}(T\mid y)\) and \(b=r_{p,2}(T\mid y)\) gives the stated bound. The approximation-cover statement follows from
		Corollary~\ref{cor:projection-induced-cover}.
	\end{proof}
	
	The symmetric average projection is used because it is linear in the
	normalized objective ratios and preserves the additive MST structure needed
	in the runtime analysis. Other projections, including weighted averages or
	max-type projections, may lead to different lifting losses but would require
	separate reachability arguments.
	
	\subsection{Runtime Bound}
	\label{subsec:bpbomst-runtime}
	
	The previous subsection shows that the symmetric average projection induces
	a common approximation cover of \(\mathcal{PF}_{\mathrm{com}}\). We now
	analyze the expected time until the common archive contains such a cover.
	The analysis follows the front-filling principle used for bi-objective MSTs.
	Local spanning-tree exchanges fill convex-front points, while recombination
	may skip convex-front segments by combining edge structures already available
	in the population.
	
	For every \(y=(y_{1,1},y_{1,2},y_{2,1},y_{2,2})\in\mathcal{PF}_{\mathrm{com}}\), the average projection induces
	\[
	A_{p,y}^{\mathrm{avg}}(T)
	=
	\frac12
	\left(
	\frac{f_1^{(p)}(T)}{y_{p,1}}
	+
	\frac{f_2^{(p)}(T)}{y_{p,2}}
	\right),
	\qquad p\in\{1,2\}.
	\]
	Since \(f_i^{(p)}(T)=\sum_{e\in T}w_i^{(p)}(e)\),
	\[
	A_{p,y}^{\mathrm{avg}}(T)
	=
	\sum_{e\in T}
	\frac12
	\left(
	\frac{w_1^{(p)}(e)}{y_{p,1}}
	+
	\frac{w_2^{(p)}(e)}{y_{p,2}}
	\right).
	\]
	Thus
	\[
	A_y^{\mathrm{avg}}(T)
	=
	\bigl(
	A_{1,y}^{\mathrm{avg}}(T),
	A_{2,y}^{\mathrm{avg}}(T)
	\bigr)
	\]
	is a genuine bi-objective MST instance with positive additive edge weights.
	
	Let \(F_y^A:=\mathcal{PF}_{A_y^{\mathrm{avg}}}\) be the Pareto front of this auxiliary bi-objective MST instance, and let \(S_y^A:=\operatorname{supp}(F_y^A)\) be the support points on the lower-left convex hull of \(F_y^A\). Let
	\(\operatorname{conv}(F_y^A)\) denote the Pareto-front points lying on this
	convex sub-front. If \(S_y^A=\{q_1^y,\ldots,q_{r_y}^y\}\) is ordered increasingly in the first auxiliary coordinate, define
	\[
	C_{y,j}^A
	=
	\left\{
	z\in\operatorname{conv}(F_y^A):
	q_{j,1}^y<z_1\le q_{j+1,1}^y
	\right\},
	\qquad
	j=1,\ldots,r_y-1.
	\]
	The total auxiliary convex-front filling size is
	\[
	C_A
	:=
	\sum_{y\in\mathcal{PF}_{\mathrm{com}}}
	\sum_{j=1}^{r_y-1}
	|C_{y,j}^A|.
	\]
	If \(C_A>0\), define
	\[
	C_{\min}^A
	:=
	\min
	\left\{
	|C_{y,j}^A|:
	y\in\mathcal{PF}_{\mathrm{com}},
	\ 1\le j<r_y,
	\ |C_{y,j}^A|>0
	\right\}.
	\]
	If \(C_A=0\), set \(C_{\min}^A=1\).
	
	\begin{definition}[Representative-compatible auxiliary fillability]
		\label{def:representative-compatible-aux-fillability}
		A BPBOMST instance satisfies representative-compatible auxiliary fillability if, for
		every \(y\in\mathcal{PF}_{\mathrm{com}}\), there exists a map \(\rho_y:\operatorname{conv}(F_y^A)\to\mathcal T\) satisfying the following conditions.
		
		\begin{enumerate}
			\item For every \(q\in\operatorname{conv}(F_y^A)\),
			\[
			A_y^{\mathrm{avg}}(\rho_y(q))=q.
			\]
			
			\item For every \(q\in\operatorname{conv}(F_y^A)\), the tree \(\rho_y(q)\)
			is the canonical joint representative of its joint objective vector
			\(Y(\rho_y(q))\).
			
			\item For every segment \(C_{y,j}^A\), the representatives of the points in
			\(C_{y,j}^A\), ordered increasingly in the first auxiliary coordinate, can be
			generated consecutively by one spanning-tree exchange from the representative
			of the preceding point. The predecessor of the first point in \(C_{y,j}^A\)
			is \(\rho_y(q_j^y)\).
			
			\item The endpoint representatives \(\rho_y(q_1^y)\) and
			\(\rho_y(q_{r_y}^y)\) are contained in the canonical party-wise
			representative populations.
		\end{enumerate}
	\end{definition}
	
	This condition is structural. It does not prescribe the search trajectory of
	CPR-NSGA-II. It only states that the auxiliary convex fronts induced by the
	instance admit the exchange-fillability needed for the runtime argument, with
	representatives compatible with the deterministic archive tie-breaking rule.
	When joint objective-vector ties are absent, the representative-compatibility
	part is automatic.
	
	For every \(y\in\mathcal{PF}_{\mathrm{com}}\), define
	\[
	\mathcal S_{\mathrm{fill}}(y)
	=
	\left\{
	\rho_y(q):
	q\in S_y^A,\ 
	A_{p,y}^{\mathrm{avg}}(\rho_y(q))\le 2,\ p\in\{1,2\}
	\right\}.
	\]
	By Theorem~\ref{thm:layered-support-cover}, this set is nonempty. Choose
	\[
	R_y^{\mathrm{fill}}
	\in
	\arg\min_{R\in\mathcal S_{\mathrm{fill}}(y)}
	L_{\mathrm{avg}}(R\mid y),
	\]
	and define
	\[
	\mathcal R_{\mathrm{fill}}
	=
	\{R_y^{\mathrm{fill}}:y\in\mathcal{PF}_{\mathrm{com}}\}.
	\]
	Let
	\[
	\lambda_{\mathrm{fill}}
	=
	\max_{y\in\mathcal{PF}_{\mathrm{com}}}
	L_{\mathrm{avg}}(R_y^{\mathrm{fill}}\mid y).
	\]
	Since
	\[
	1\le L_{\mathrm{avg}}(T\mid y)\le 2
	\]
	for all feasible trees \(T\) and all \(y\in\mathcal{PF}_{\mathrm{com}}\),
	\[
	1\le \lambda_{\mathrm{fill}}\le 2.
	\]
	Therefore, \(\mathcal R_{\mathrm{fill}}\) is a
	\(2\lambda_{\mathrm{fill}}\)-common approximation cover of
	\(\mathcal{PF}_{\mathrm{com}}\). If the representative-compatible auxiliary
	fronts contain the witnesses selected in the definition of
	\(\lambda_{\mathrm{eff}}\), then
	\(\lambda_{\mathrm{fill}}=\lambda_{\mathrm{eff}}\).
	
	For each party \(p\), let
	\[
	\mathcal{PF}_p
	=
	\{F_p(T):T\in\mathcal{PS}_p\}
	\]
	be the party-wise bi-objective Pareto front. Let
	\(\operatorname{conv}(\mathcal{PF}_p)\) denote the party-wise convex
	sub-front. Define
	\[
	C_{\mathrm{pw}}
	:=
	|\operatorname{conv}(\mathcal{PF}_1)|
	+
	|\operatorname{conv}(\mathcal{PF}_2)|.
	\]
	Let \(P_{\max}\) be an upper bound on the size of the provider sampling
	pool. In the representative-pool variant of Algorithm~\ref{alg:cnsga-ii-m3st},
	each party-wise population contains at most \(W+1\) nondominated objective-vector representatives, and hence \(P_{\max}\le 2(W+1)\).
	
	\begin{definition}[CPR-good segment]
		\label{def:good-cpr-segment}
		Consider an auxiliary segment \(C_{y,j}^A\). Let \(Z_{y,j}^{\mathrm L}:=\rho_y(q_j^y)\) and \(Z_{y,j}^{\mathrm R}:=\rho_y(q_{j+1}^y)\).
		The segment \(C_{y,j}^A\) is CPR-good if there exists a canonical party-wise
		representative \(B\) such that
		\[
		Z_{y,j}^{\mathrm R}
		\in
		\mathcal T\!\left(G[Z_{y,j}^{\mathrm L}\cup B]\right).
		\]
		For a CPR-good segment, define
		\[
		\omega_{y,j}
		:=
		\min_B
		\left|
		\mathcal T\!\left(G[Z_{y,j}^{\mathrm L}\cup B]\right)
		\right|,
		\]
		where the minimum is taken over all party-wise representatives \(B\)
		satisfying the reachability condition above.
	\end{definition}
	
	Let
	\[
	\mathcal G_{\mathrm{CPR}}
	=
	\{(y,j):C_{y,j}^A\text{ is CPR-good}\}.
	\]
	Define
	\[
	N_{\mathrm{CPR}}
	=
	|\mathcal G_{\mathrm{CPR}}|,
	\qquad
	G_{\mathrm{CPR}}
	=
	\sum_{(y,j)\in\mathcal G_{\mathrm{CPR}}}
	|C_{y,j}^A|,
	\]
	and
	\[
	\Omega_{\mathrm{CPR}}
	=
	\sum_{(y,j)\in\mathcal G_{\mathrm{CPR}}}
	\omega_{y,j}.
	\]
	Since every CPR-good segment has size at least \(C_{\min}^A\),
	\[
	G_{\mathrm{CPR}}
	\ge
	N_{\mathrm{CPR}}C_{\min}^A.
	\]
	The quantity \(G_{\mathrm{CPR}}\) measures the amount of auxiliary local
	filling bypassed by edge-union recombination, while \(\Omega_{\mathrm{CPR}}\) measures
	the repair ambiguity paid by UniformRepair. Figure~\ref{fig:bpbomst-cpr-shortcut}
	illustrates a non-vacuous CPR-good shortcut. The example is used only to
	visualize the structural meaning of the parameters; it is not an additional
	assumption in the proof.
	
	\begin{figure}[H]
		\centering
		\includegraphics[width=1.0\linewidth]{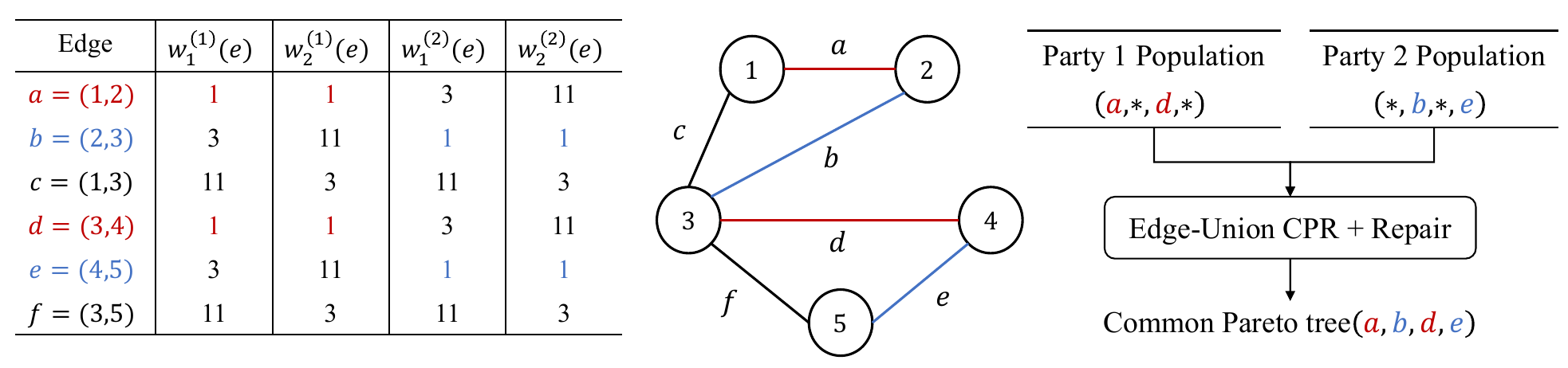}
		\caption{A five-node BPBOMST instance illustrating a CPR-good shortcut.
			Edges \(a\) and \(d\) are favored by Party~1, while edges \(b\) and \(e\)
			are favored by Party~2. The notation \((a,*,d,*)\) denotes the displayed
			Party-1 representatives that preserve the Party-1-favored edges, and
			\((*,b,*,e)\) denotes the displayed Party-2 representatives that preserve
			the Party-2-favored edges. These representatives do not by themselves contain
			all edges of the target common Pareto-optimal tree \(T^\star=(a,b,d,e)\); each misses
			the other party's key edges. Edge-union recombination followed by repair can assemble
			these complementary edge subsets and generate \(T^\star\), illustrating how
			CPR can create a shortcut relevant to common-cover construction.}
		\label{fig:bpbomst-cpr-shortcut}
	\end{figure}
	
	\begin{lemma}[Auxiliary Pareto representatives are joint-nondominated]
		\label{lem:auxiliary-pareto-joint-nondominated}
		Let \(T\) be Pareto-optimal for the auxiliary objective
		\(A_y^{\mathrm{avg}}\). Then no feasible spanning tree \(T'\) strictly
		dominates \(T\) under the joint objective vector \(Y\).
	\end{lemma}
	
	\begin{proof}
		Assume that there exists \(T'\) such that
		\[
		Y(T')\le Y(T)
		\]
		component-wise and at least one inequality is strict. Then, for each party
		\(p\),
		\[
		f_1^{(p)}(T')\le f_1^{(p)}(T),
		\qquad
		f_2^{(p)}(T')\le f_2^{(p)}(T).
		\]
		Since all components of \(y\) are positive,
		\[
		A_{p,y}^{\mathrm{avg}}(T')
		\le
		A_{p,y}^{\mathrm{avg}}(T),
		\qquad p\in\{1,2\}.
		\]
		At least one of the two auxiliary inequalities is strict. Hence
		\[
		A_y^{\mathrm{avg}}(T')
		\prec
		A_y^{\mathrm{avg}}(T),
		\]
		contradicting the Pareto optimality of \(T\) under \(A_y^{\mathrm{avg}}\).
	\end{proof}
	
	\begin{lemma}[Archive preservation of cover certificates]
		\label{lem:archive-preservation-fill}
		Suppose the common archive contains a tree \(T\) satisfying
		\[
		Y(T)\le Y(R)
		\]
		for some witness \(R\). Then subsequent archive updates cannot destroy the
		certificate represented by \(R\).
	\end{lemma}
	
	\begin{proof}
		If \(T\) is removed from the common archive, then either another archived tree \(T'\) satisfies \(Y(T')\le Y(T)\), or \(T\) is replaced by a canonical representative with the same joint objective vector. In the first case, \(Y(T')\le Y(T)\le Y(R)\). In the second case, the replacing tree \(T'\) satisfies \(Y(T')=Y(T)\), and hence again \(Y(T')\le Y(R)\). Therefore the archive still contains a tree no worse than \(R\) in all four objectives.
	\end{proof}
	
	\begin{lemma}[Party-wise convex-front generation]
		\label{lem:partywise-conv-front-generation}
		The canonical party-wise convex-front representatives needed by the provider
		pool can be generated and retained within expected
		\[
		O\!\left(m^2W\log W\cdot C_{\mathrm{pw}}\right)
		\]
		fitness evaluations.
	\end{lemma}
	
	\begin{proof}
		For each party \(p\), the party-wise problem is a bi-objective MST with
		positive integer objective values bounded by \(W\). The standard
		bi-objective MST convex-front filling argument~\cite{cerf2023} gives exchange paths along
		the party-wise convex sub-front. In one iteration, selecting a specified
		predecessor from a party-wise representative population has probability
		\(\Omega(1/W)\), and sampling a specified spanning-tree exchange has
		probability \(\Omega(1/m^2)\). The conservative \(O(\log W)\) factor accounts
		for bounded integer fitness-level progress from arbitrary initialization to
		the relevant party-wise representative level. Thus one party-wise
		convex-front progress point costs
		\[
		O(m^2W\log W)
		\]
		expected evaluations. Summing over the \(C_{\mathrm{pw}}\) party-wise
		convex-front representatives gives the bound. The representative-pool update
		retains the canonical representative of every party-wise nondominated
		objective vector once generated.
	\end{proof}
	
	Let \(A_{\max}\) be an upper bound on the number of joint objective-vector
	representatives in the common archive. Since every joint objective vector
	has four positive integer components bounded by \(W\), a nondominated set in
	\([W]^4\) contains at most one point for every fixed choice of the first
	three coordinates. We may use
	\[
	A_{\max}=O(W^3).
	\]
	
	\begin{lemma}[Second-layer auxiliary filling with CPR shortcuts]
		\label{lem:second-layer-filling-instance}
		Assume the BPBOMST instance is representative-compatible
		auxiliary-fillable. After the party-wise convex-front representatives have
		been generated, the expected number of additional fitness evaluations until
		the common archive contains a tree no worse than every witness in
		\(\mathcal R_{\mathrm{fill}}\) is at most
		\[
		O\!\left(
		\frac{A_{\max}m^2}{1-p_g}
		\left(C_A-G_{\mathrm{CPR}}\right)
		+
		\frac{A_{\max}P_{\max}}{p_g}
		\Omega_{\mathrm{CPR}}
		\right).
		\]
		Equivalently,
		\[
		O\!\left(
		\frac{A_{\max}m^2}{1-p_g}
		\left(C_A-N_{\mathrm{CPR}}C_{\min}^A\right)
		+
		\frac{A_{\max}P_{\max}}{p_g}
		\Omega_{\mathrm{CPR}}
		\right)
		\]
		is an upper bound.
	\end{lemma}
	
	\begin{proof}
		Fix \(y\in\mathcal{PF}_{\mathrm{com}}\) and a segment \(C_{y,j}^A\).
		Representative-compatible auxiliary fillability gives an exchange ordering
		of the representatives on this segment. This ordering is used only to
		identify successful events in the analysis; the algorithm does not need to
		know \(y\), \(A_y^{\mathrm{avg}}\), or the ordering.
		
		Consider a local progress point not bypassed by a CPR-good segment. Once its
		predecessor representative is present in the common archive, selecting this
		receiver has probability at least \(\Omega(1/A_{\max})\). The local archive
		step is chosen with probability \(1-p_g\). A specified one-edge exchange is
		sampled with probability \(\Omega(1/m^2)\). Hence the probability of
		realizing this local progress event in one iteration is at least
		\[
		\Omega\!\left(
		\frac{1-p_g}{A_{\max}m^2}
		\right),
		\]
		and its expected waiting time is
		\[
		O\!\left(
		\frac{A_{\max}m^2}{1-p_g}
		\right).
		\]
		There are \(C_A-G_{\mathrm{CPR}}\) auxiliary local progress points not
		bypassed by CPR-good segments, giving
		\[
		O\!\left(
		\frac{A_{\max}m^2}{1-p_g}
		\left(C_A-G_{\mathrm{CPR}}\right)
		\right).
		\]
		
		Now consider a CPR-good segment \((y,j)\in\mathcal G_{\mathrm{CPR}}\). Let \(Z_{y,j}^{\mathrm L}=\rho_y(q_j^y)\) and \(Z_{y,j}^{\mathrm R}=\rho_y(q_{j+1}^y)\), and let \(B\) be a party-wise provider attaining \(\omega_{y,j}\). Once
		\(Z_{y,j}^{\mathrm L}\) is present in the common archive, selecting it as the
		receiver has probability at least \(\Omega(1/A_{\max})\). Selecting \(B\)
		from the provider pool has probability at least \(\Omega(1/P_{\max})\). The
		CPR step is chosen with probability \(p_g\). Since the segment is CPR-good,
		\[
		Z_{y,j}^{\mathrm R}
		\in
		\mathcal T(G[Z_{y,j}^{\mathrm L}\cup B]).
		\]
		UniformRepair returns \(Z_{y,j}^{\mathrm R}\) with probability at least
		\(1/\omega_{y,j}\). The expected waiting time for this CPR event is
		\[
		O\!\left(
		\frac{A_{\max}P_{\max}}{p_g}\omega_{y,j}
		\right).
		\]
		Summing over all CPR-good segments gives
		\[
		O\!\left(
		\frac{A_{\max}P_{\max}}{p_g}
		\Omega_{\mathrm{CPR}}
		\right).
		\]
		
		Every representative on an auxiliary convex front is Pareto-optimal under
		\(A_y^{\mathrm{avg}}\). By
		Lemma~\ref{lem:auxiliary-pareto-joint-nondominated}, such a representative
		cannot be strictly dominated under \(Y\). If another tree with the same joint
		objective vector is encountered, the canonical archive rule keeps the
		representative specified by
		Definition~\ref{def:representative-compatible-aux-fillability}. Thus the
		edge structure required for subsequent local or CPR progress is retained.
		Together with Lemma~\ref{lem:archive-preservation-fill}, the generated
		representatives suffice to preserve the cover certificates. The second
		displayed bound follows from
		\(G_{\mathrm{CPR}}\ge N_{\mathrm{CPR}}C_{\min}^A\).
	\end{proof}
	
	\begin{theorem}[Instance-parameterized runtime bound for BPBOMST]
		\label{thm:runtime-mpmomst-instance}
		Consider the representative-pool version of CPR-NSGA-II for BPBOMST with
		CPR probability \(p_g\in(0,1)\), common archive update under joint
		nondominance, and UniformRepair for edge-union recombination. Assume the BPBOMST
		instance is representative-compatible auxiliary-fillable. Then the expected
		number of fitness evaluations until the common archive contains a
		\(2\lambda_{\mathrm{fill}}\)-common approximation cover of
		\(\mathcal{PF}_{\mathrm{com}}\) is at most
		\[
		O\!\left(
		m^2W\log W\cdot C_{\mathrm{pw}}
		+
		\frac{A_{\max}m^2}{1-p_g}
		\left(C_A-G_{\mathrm{CPR}}\right)
		+
		\frac{A_{\max}P_{\max}}{p_g}
		\Omega_{\mathrm{CPR}}
		\right).
		\]
		Equivalently,
		\[
		O\!\left(
		m^2W\log W\cdot C_{\mathrm{pw}}
		+
		\frac{A_{\max}m^2}{1-p_g}
		\left(C_A-N_{\mathrm{CPR}}C_{\min}^A\right)
		+
		\frac{A_{\max}P_{\max}}{p_g}
		\Omega_{\mathrm{CPR}}
		\right)
		\]
		is an upper bound. With \(A_{\max}=O(W^3)\) and \(P_{\max}\le 2(W+1)\), we obtain
		\[
		O\!\left(
		m^2W\log W\cdot C_{\mathrm{pw}}
		+
		\frac{m^2W^3}{1-p_g}
		\left(C_A-N_{\mathrm{CPR}}C_{\min}^A\right)
		+
		\frac{W^4}{p_g}
		\Omega_{\mathrm{CPR}}
		\right).
		\]
		Since \(1\le\lambda_{\mathrm{fill}}\le2\), the obtained archive is always a
		\(4\)-common approximation cover.
	\end{theorem}
	
	\begin{proof}
		Lemma~\ref{lem:partywise-conv-front-generation} gives
		\[
		O\!\left(m^2W\log W\cdot C_{\mathrm{pw}}\right)
		\]
		expected evaluations for generating and retaining the party-wise
		convex-front provider representatives. The endpoint representatives needed
		to start the auxiliary filling processes are submitted to the common archive
		when generated. After these representatives are available,
		Lemma~\ref{lem:second-layer-filling-instance} bounds the expected additional
		time until the common archive contains a tree no worse than every witness in
		\(\mathcal R_{\mathrm{fill}}\). By construction,
		\(\mathcal R_{\mathrm{fill}}\) is a
		\(2\lambda_{\mathrm{fill}}\)-common approximation cover of
		\(\mathcal{PF}_{\mathrm{com}}\), and
		Lemma~\ref{lem:archive-preservation-fill} ensures that the certificates are
		not destroyed by later archive updates. Summing the two phases gives the
		first bound. The other bounds follow from
		\(G_{\mathrm{CPR}}\ge N_{\mathrm{CPR}}C_{\min}^A\),
		\(A_{\max}=O(W^3)\), and \(P_{\max}\le 2(W+1)\).
	\end{proof}
	
	\begin{corollary}
		\label{cor:witness-compatible-runtime}
		If, for every \(y\in\mathcal{PF}_{\mathrm{com}}\), the representative-compatible auxiliary filling structure contains the witness \(R_y^\star\) selected in the definition of \(\lambda_{\mathrm{eff}}\), then \(\lambda_{\mathrm{fill}}=\lambda_{\mathrm{eff}}\), and the runtime bound in Theorem~\ref{thm:runtime-mpmomst-instance} holds for a \(2\lambda_{\mathrm{eff}}\)-common approximation cover.
	\end{corollary}
	
	\begin{proof}
		Under the stated condition, the fill witnesses selected in the representative-compatible auxiliary fronts coincide with the witnesses defining \(\lambda_{\mathrm{eff}}\). Hence \(\lambda_{\mathrm{fill}}=\lambda_{\mathrm{eff}}\). The claimed runtime bound follows directly from Theorem~\ref{thm:runtime-mpmomst-instance}.
	\end{proof}
	
	The parameters in Theorem~\ref{thm:runtime-mpmomst-instance} are all
	instance-structural. The term \(C_A\) is the total size of the auxiliary
	convex fronts induced by the average projections over all
	\(y\in\mathcal{PF}_{\mathrm{com}}\), and it plays the role of
	\(|\operatorname{conv}(F^*)|\) in the bi-objective MST analysis. The term
	\(C_{\min}^A\) is the smallest nonempty auxiliary convex-segment size.
	The term \(N_{\mathrm{CPR}}\) counts the auxiliary convex segments whose
	right endpoints are reachable from their left endpoints by edge-union recombination,
	while \(G_{\mathrm{CPR}}\) records the exact local-filling gain. The term
	\(\Omega_{\mathrm{CPR}}\) is the total spanning-tree ambiguity paid by
	UniformRepair. The term \(C_{\mathrm{pw}}\) is the total size of the
	party-wise convex fronts that supply provider edge structures.
	
	A CPR-good segment of size \(s\) and ambiguity \(\omega\) is useful when the
	local filling time it saves dominates the CPR time it costs. Ignoring
	constant factors, this requires
	\[
	\frac{A_{\max}m^2}{1-p_g}s
	\gtrsim
	\frac{A_{\max}P_{\max}}{p_g}\omega,
	\]
	or equivalently
	\[
	s
	\gtrsim
	\frac{1-p_g}{p_g}
	\cdot
	\frac{P_{\max}}{m^2}
	\cdot
	\omega.
	\]
	Thus the theorem predicts a speed-up on instances where edge-union recombination can
	jump over long auxiliary convex segments and the corresponding union
	subgraphs have moderate spanning-tree ambiguity.
	
	\subsection{Implications of the BPBOMST Analysis}
	\label{subsec:bpbomst-implications}
	
	The BPBOMST analysis clarifies why common-solution search should not be
	reduced either to independent party-wise optimization or to a flattened
	many-objective formulation. Both baselines are natural, but neither is
	aligned with the common Pareto structure targeted here.
	
	A party-wise baseline solves the bi-objective MST problem of each party
	independently and then intersects the returned approximation sets. Suppose
	party \(p\) returns a set \(\mathcal A_p\) that approximates
	\(\mathcal{PS}_p\). The post-processing rule would form \(\mathcal A_{\cap}=\mathcal A_1\cap\mathcal A_2\).
	This strategy is simple, but a party-wise approximation guarantee only
	states that each party-wise Pareto tree can be approximated under the
	objectives of the same party. It does not require the approximating trees
	selected for different parties to coincide.
	
	\begin{proposition}[Independent approximation is insufficient]
		\label{prop:independent-insufficient}
		For BPBOMST, party-wise approximation guarantees for \(\mathcal{PS}_1\) and
		\(\mathcal{PS}_2\) alone do not imply that
		\(\mathcal A_1\cap\mathcal A_2\) contains a common approximation cover of
		\(\mathcal{PF}_{\mathrm{com}}\).
	\end{proposition}
	
	\begin{proof}
		The party-wise guarantee is defined separately for each \(F_p\). For a
		common Pareto-optimal tree \(T^\star\in\mathcal{PS}_{\mathrm{com}}\), party \(1\) may
		approximate \(T^\star\) by a tree \(T_1\) that is good under \(F_1\), while
		party \(2\) may approximate the same \(T^\star\) by a different tree \(T_2\)
		that is good under \(F_2\). The definition imposes no requirement that
		\(T_1=T_2\), nor does it require that the two approximation sets overlap on a
		tree that approximates \(T^\star\) under both parties. Hence the intersection
		may be empty. Even when it is nonempty, membership in both party-wise
		approximation sets does not imply approximation of the same common Pareto
		point under the four components of \(Y\).
	\end{proof}
	
	Flattening takes the opposite route. It combines all party objectives into a
	single high-dimensional vector. For BPBOMST, the flattened objective is
	\[
	\mathbf f_{\mathrm{flat}}(T)
	=
	\bigl(
	f_1^{(1)}(T),f_2^{(1)}(T),
	f_1^{(2)}(T),f_2^{(2)}(T)
	\bigr)
	=
	Y(T).
	\]
	For a general MPMOMST instance with \(M\) parties and \(k_p\) objectives for party \(p\), flattening gives \(K=\sum_{p=1}^{M}k_p\) objectives. Flattening preserves all objective components, but it changes
	the search object. The flattened Pareto set is the nondominated set under
	the full \(K\)-dimensional objective vector, whereas BPBOMST targets \(\mathcal{PF}_{\mathrm{com}}=\{Y(T):T\in\mathcal{PS}_1\cap\mathcal{PS}_2\}\).
	
	\begin{proposition}[Archive-size pressure under flattening]
		\label{prop:flattening-archive-pressure}
		In a \(K\)-objective flattened formulation with positive integer objective
		values bounded by \(W\), a front-preserving elitist analysis must allow for
		mutually nondominated objective vectors whose number scales as
		\[
		\Omega(W^{K-1})
		\]
		in the worst case.
	\end{proposition}
	
	\begin{proof}
		Each objective coordinate takes integer values in \(\{1,\ldots,W\}\). In a \(K\)-dimensional objective lattice, all points with the same coordinate sum form an antichain under component-wise minimization. Choosing a middle level of the lattice gives an antichain of size \(\Omega(W^{K-1})\). An elitist algorithm that attempts to preserve all relevant nondominated flattened objective vectors must therefore provide population or archive capacity proportional to the size of such an antichain.
	\end{proof}
	
	This capacity statement is not a claim that every BPBOMST instance realizes	the largest possible antichain. Its role is to explain why a flattened front-preserving analysis is controlled by the ambient objective dimension \(K\). The BPBOMST bound in Theorem~\ref{thm:runtime-mpmomst-instance} is controlled instead by \(C_A\), \(C_{\min}^A\), \(N_{\mathrm{CPR}}\), \(\Omega_{\mathrm{CPR}}\), and \(C_{\mathrm{pw}}\). These quantities describe auxiliary common-cover geometry, party-wise
	provider fronts, and edge-union reachability.
	
	The role of CPR is therefore not merely to add diversity. Party-wise search generates edge structures that are good under individual bi-objective evaluations. The common archive records progress under \(Y\). Edge-union recombination	connects these two layers by injecting party-wise edge structures into archive receivers. Independent search preserves party-wise quality but does not coordinate the representatives needed for common approximation. Flattened search coordinates all objectives but loses the party structure
	that makes cross-party edge transfer meaningful. CPR-NSGA-II occupies the intermediate position. It keeps party-wise structure available and uses the common archive to test whether such structure can be assembled into common approximate solutions.
	
	The main implication is that the relevant object in multi-party	combinatorial search is neither the flattened Pareto front nor the union of independently approximated party-wise fronts. It is the interaction between the common Pareto structure and the cross-party recombination structure. The former determines what should be covered, while the latter determines whether party-wise components can be assembled efficiently into common approximate
	solutions.
	
	\section{Experiments}
	\label{sec:experiments}
	
	In this section, we empirically evaluate CPR-NSGA-II on the MP-JCG and BPBOMST instances. Rather than providing an exhaustive empirical benchmark, these experiments are designed to corroborate our theoretical analyses. Specifically, they aim to verify whether the observed search dynamics align with the mechanisms established in the previous sections, thereby illustrating the effect of cross-party recombination on consensus search.
	
	\subsection{Consensus Search on MP-JCG}
	\label{subsec:experiments-mpjcg}
	
	To examine the runtime trend established in Section~\ref{sec:mpjcg-analysis}, we record the number of fitness evaluations (FEs) required to identify the common Pareto set across varying problem sizes. Specifically, we set the gap parameter \(k=3\) and vary the decision dimension \(n \in \{10,20,\dots,100\}\). For each configuration, we perform 10 independent runs with a maximum evaluation budget of \(FE_{\max}=1{,}000{,}000\). In the CPR-based implementation (denoted as CNSGA-II-IS), the population size is fixed at 50, the mutation rate is set to \(1/n\), and the inter-party crossover probability is set to 0.5. This implementation closely mirrors the analytical CPR-NSGA-II variant, with the exception of adopting practical finite-population parameters.
	
	Figure~\ref{fig:mpjcg-fe} compares the FEs of CNSGA-II-IS with those of the payoff-guided mutation baseline, \textsc{GEMPMO\_payoff}. The empirical results are consistent with the theoretical prediction: cross-party recombination effectively reduces the dominant waiting time by assembling complementary components, thereby bypassing the rare mutation events required by the baseline. % A one-sided Mann--Whitney \(U\) test indicates that the reduction in FEs is statistically significant at the 5\% level for all tested dimensions \(n\ge20\).
	
	\begin{figure}[H]
		\centering
		\includegraphics[width=0.7\linewidth]{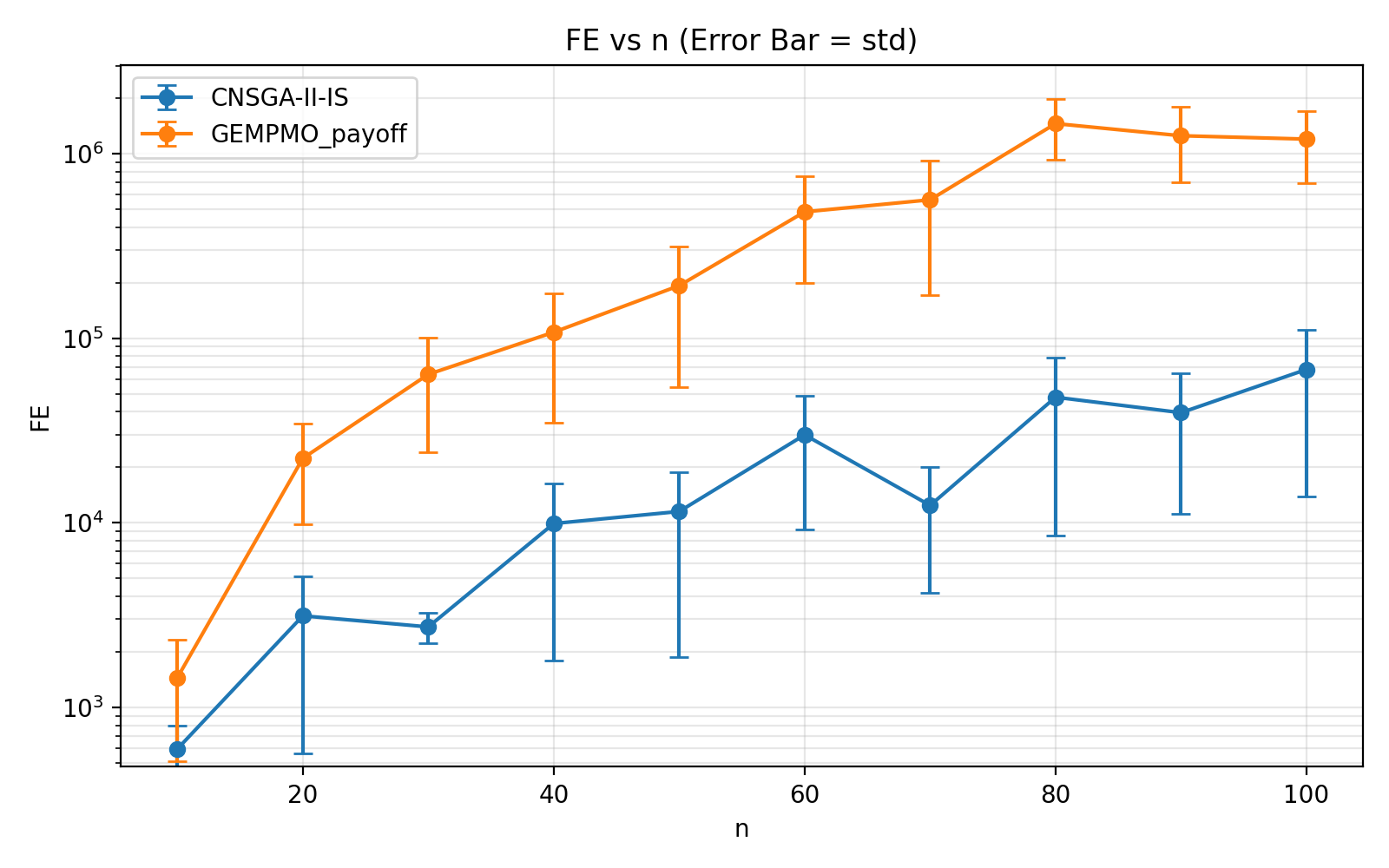}
		\caption{Fitness evaluations required to identify the common Pareto set on MP-JCG instances. Error bars represent one standard deviation over 10 independent runs.}
		\label{fig:mpjcg-fe}
	\end{figure}
	
	\subsection{Consensus Search on BPBOMST}
	\label{subsec:experiments-bpbomst}
	
	For the BPBOMST instances, we evaluate the efficiency of reaching prescribed common approximation ratios with respect to a known common Pareto front. Each instance is generated by initially specifying a set of common Pareto-optimal spanning trees, and subsequently constructing the graph topology and party-wise edge weights such that these trees constitute the target common Pareto front. This controlled setting allows the approximation ratio to be evaluated directly against a known consensus target, which aligns with the common-cover hitting-time perspective discussed in Section~\ref{sec:bpbomst-analysis}. We conduct experiments on randomly generated bi-party, two-objective-per-party undirected graphs with node sizes \(n\in\{5,\dots,29\}\). The graphs are generated from a base path topology augmented with random edges, capping the total number of edges at \(2n\). For each problem size, we perform 5 independent runs with a dynamic evaluation budget of \(FE_{\max}=20{,}000n\). The population size is set to \(100(n-1)+1\), and edge weights are assigned uniformly at random. We record the FEs required by CPR-NSGA-II and the independent party-wise baseline, CNSGA-II-Par, to achieve approximation ratios of \(\alpha\in\{2,3,4\}\).
	
	Figure~\ref{fig:mpmomst-approx-ratio} reports the FEs required to achieve the target approximation ratios across different problem sizes. On the evaluated instances, CPR-NSGA-II attains the target ratios utilizing fewer evaluations than the independent party-wise baseline. This observation corroborates the theoretical utility of the common archive and edge-union recombination. Specifically, while the independent baseline may improve individual party-wise fronts, it does not guarantee the generation of trees that simultaneously approximate the common Pareto objective vector for both parties. In contrast, CPR-NSGA-II leverages the common archive to track joint objective-vector progress and utilizes edge-union recombination to transfer party-wise edge structures into common-archive receivers.
	
	\begin{figure}[H]
		\centering
		\subfloat[\(\alpha=2\)]{\includegraphics[width=0.32\linewidth]{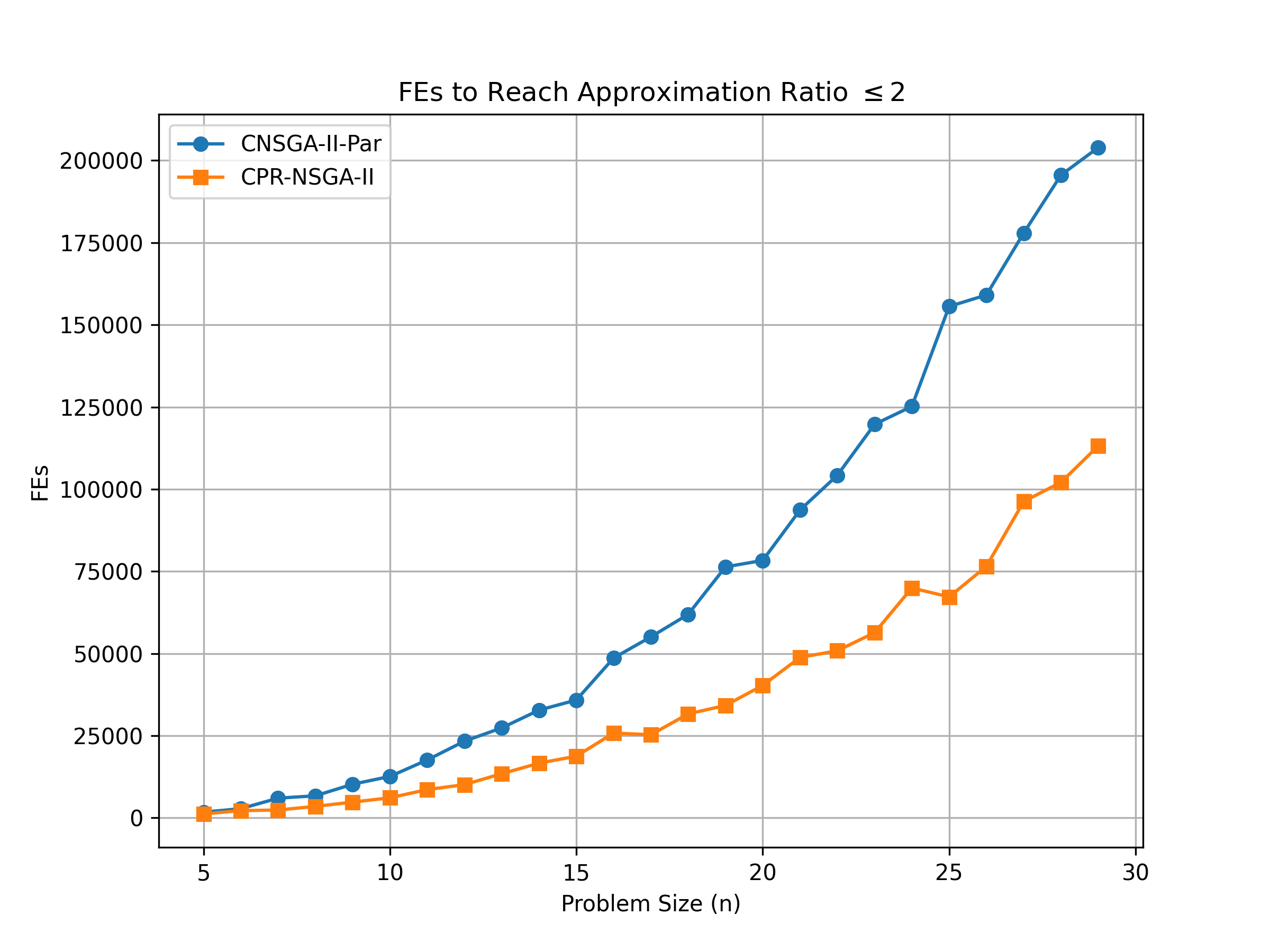}}\hfill
		\subfloat[\(\alpha=3\)]{\includegraphics[width=0.32\linewidth]{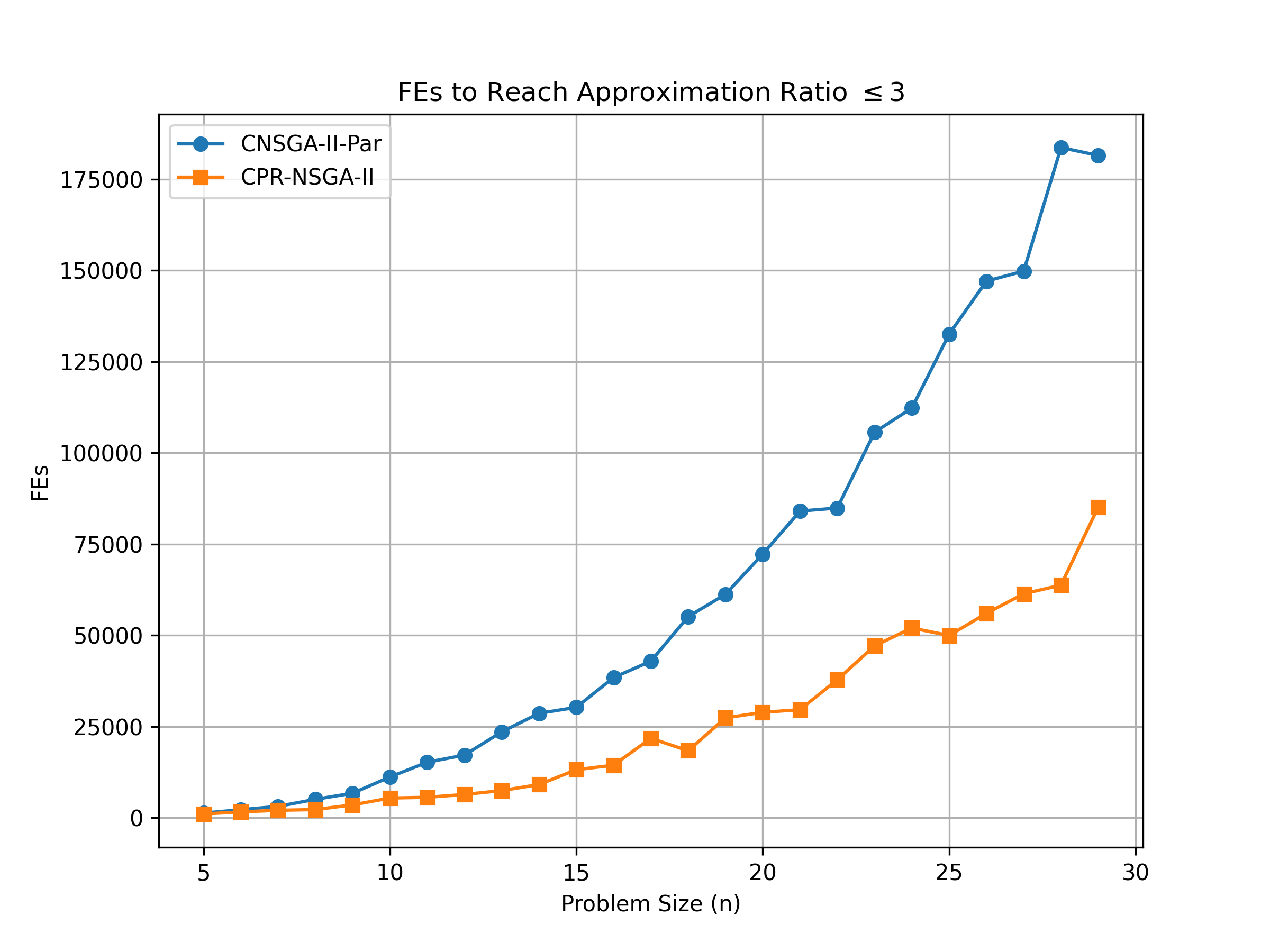}}\hfill
		\subfloat[\(\alpha=4\)]{\includegraphics[width=0.32\linewidth]{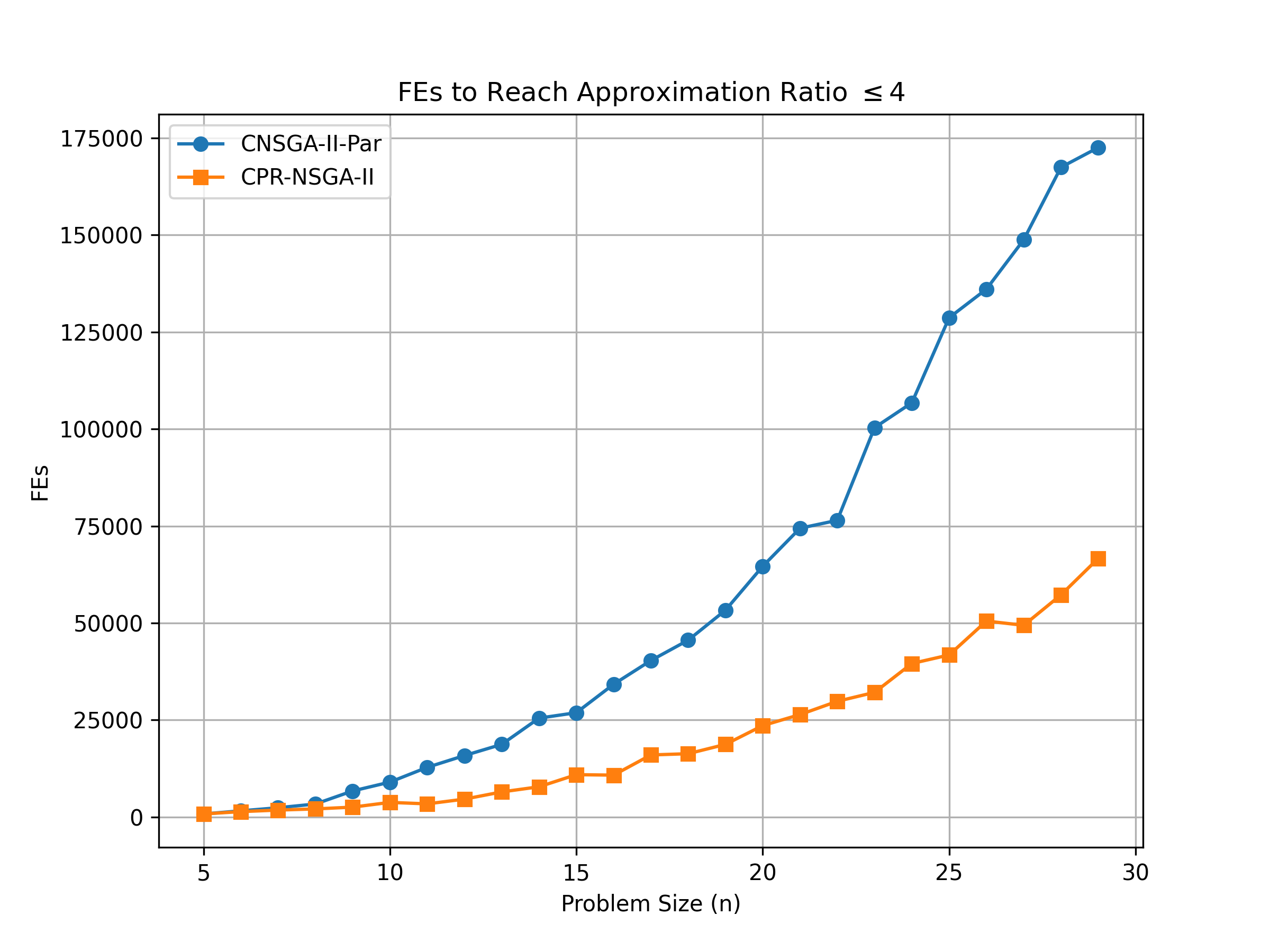}}
		\caption{Fitness evaluations required to achieve target approximation ratios \(\alpha\in\{2,3,4\}\) on BPBOMST instances. The panels correspond to \(\alpha=2\), \(\alpha=3\), and \(\alpha=4\), respectively.}
		\label{fig:mpmomst-approx-ratio}
	\end{figure}
	
	The results from both experimental settings support a unified qualitative conclusion across different search spaces. In MP-JCG, cross-party recombination merges complementary bit patterns from different parties, thereby expediting the convergence to the common Pareto set. In BPBOMST, it combines edge structures preserved during party-wise search, validating their effectiveness through a common archive. These findings align with the theoretical mechanisms identified in Sections~\ref{sec:mpjcg-analysis} and~\ref{sec:bpbomst-analysis}, demonstrating that cross-party recombination is beneficial when party-wise search maintains components capable of being assembled into common candidate solutions.
	
	\section{Discussion and Conclusions}
	\label{sec:conclusion}
	
	This paper investigated multi-party multi-objective optimization (MPMOP) from a consensus-search perspective, theoretically analyzing how cross-party recombination alters dominant search bottlenecks. Our analysis distinguishes multi-party consensus search from conventional many-objective optimization. Specifically, flattening party-wise objectives into a single high-dimensional vector preserves all objective components but shifts the search target from the common Pareto structure to a more expansive nondominated set in the ambient space. Conversely, maintaining party-wise populations preserves heterogeneous search information, allowing cross-party recombination to assemble compatible components into common candidate solutions.
	
	On the pseudo-Boolean benchmark MP-JCG, we proved that a payoff-guided mutation baseline incurs an expected fitness-evaluation complexity of \(\Theta(n^2)\). In contrast, an analytical CPR-NSGA-II variant discovers both common Pareto-optimal solutions in \(O(n\log n)\) expected evaluations under the stated assumptions. This efficiency gain stems from replacing a rare mutation-driven gap-crossing event with the direct assembly of complementary party-wise components: Party~2 supplies a nearly complete prefix, Party~1 provides a constant-length suffix template, and a boundary crossover merges them into a safe intermediate solution.
	
	For the BPBOMST problem (a bi-party, two-objective-per-party spanning tree setting), we developed a layered support-cover analysis. The symmetric average projection of each common Pareto objective vector induces an auxiliary bi-objective MST instance. Suitable support representatives yield a party-level factor-2 cover, and lifting this bound to the original four objectives produces a \(2\lambda\)-common approximation guarantee, where \(\lambda\in[1,2]\). Furthermore, we derived an instance-parameterized expected runtime bound for a representative-pool CPR-NSGA-II variant under representative-compatible auxiliary fillability, edge-union reachability, and uniform repair. By separating the effects of auxiliary common-cover geometry, party-wise provider-front size, CPR-good shortcut gains, and edge-union repair ambiguity, this bound theoretically justifies when cross-party recombination accelerates common-cover construction, advancing beyond its treatment as a black-box heuristic.
	
	Our empirical observations corroborate these theoretical mechanisms. On MP-JCG, the CPR-based implementation demands fewer fitness evaluations to identify the common Pareto set compared to the payoff-guided baseline. On BPBOMST, CPR-NSGA-II achieves target common approximation ratios faster than the independent party-wise baseline. While the independent baseline improves individual bi-objective fronts, it fails to ensure that the resulting representatives coincide on common approximate solutions. CPR-NSGA-II mitigates this by preserving party-wise structures and employing a common archive to verify whether these structures can be assembled to approximate the common Pareto front.
	
	These findings also elucidate why multi-party consensus search should not be conflated with conventional many-objective optimization. Although flattening coordinates all components, it can unnecessarily expand the search target, increasing the population or archive capacity required by front-preserving elitist algorithms, and dispersing search efforts across trade-off regions irrelevant to the common Pareto set. The proposed CPR framework circumvents this by keeping the party-wise structure explicit and leveraging recombination to aggregate cross-party complementary information.
	
	The provable advantages of cross-party recombination inherently depend on the structural properties of the problem instances, particularly cross-party complementarity. If party-wise objectives are highly aligned or the search space lacks recombinable substructures, recombination may prove neutral or even disruptive. Several analytical limitations also delineate the scope of the current theory. First, the MP-JCG analysis assumes a constant gap parameter; scaling this parameter with problem size may alter the complexity separation between recombination-based and mutation-only search. Second, the \(O(n\log n)\) bound for MP-JCG relies on an analytical CPR-NSGA-II variant that isolates recombination effects via structured template construction and random template availability. A more comprehensive theory should derive template availability directly from standard population dynamics. Third, the BPBOMST analysis is confined to the bi-party, two-objective-per-party regime and assumes representative-compatible auxiliary fillability, uniform edge-union repair, and instance-dependent reachability parameters. Finally, the limitations identified for the flattening approach are primarily based on worst-case capacity arguments rather than tight hitting-time lower bounds.
	
	Several avenues remain open for future research. One direction is to extend the BPBOMST framework to scenarios with more parties and additional objectives per party. Another is to derive runtime bounds explicitly in terms of complementarity parameters, the population dynamics that preserve useful party-wise structures, and the repair ambiguity introduced by edge-union recombination.
	
	\bibliographystyle{model2-names.bst}
	\bibliography{AIJ}
	
\end{document}